\documentclass[12pt]{article}
\usepackage[margin=1in]{geometry}

\usepackage[authoryear]{natbib}
\usepackage[utf8]{inputenc} 
\usepackage[T1]{fontenc}    
\usepackage{hyperref}       
\usepackage{url}            
\usepackage{booktabs}       
\usepackage{amsfonts}       
\usepackage{nicefrac}       
\usepackage{microtype}      
\usepackage[pdftex]{graphicx}  
\usepackage{float}
\usepackage{amsthm}
\usepackage{mathtools}
\usepackage{amssymb}
\floatstyle{ruled}
\usepackage{xcolor}

\usepackage{adjustbox}
\usepackage{wrapfig}
\usepackage{caption}
\usepackage{subcaption}
\usepackage{cprotect}
\usepackage{etoc}
\usepackage{titlesec}
\usepackage{enumitem}
\usepackage{sectsty}
\sectionfont{\centering}

\newcommand{\indep}{\raisebox{0.05em}{\rotatebox[origin=c]{90}{$\models$}}}

\theoremstyle{definition}\newtheorem{theorem}{Theorem}
\newtheorem{assumption}{Assumption}
\newtheorem{proposition}{Proposition}
\newtheorem{definition}{Definition}

\newtheorem{corollary}{Corollary}
\newtheorem{lemma}{Lemma}

\newtheorem{algorithm}{Algorithm}
\DeclareMathOperator*{\argmin}{\arg\!\min}

\newcommand{\blind}{1}

\graphicspath{{./art/}}

\begin{document}

\def\spacingset#1{\renewcommand{\baselinestretch}%
{#1}\small\normalsize} \spacingset{1}


\if1\blind
{
  \title{\bf Double Robustness for Complier Parameters and a Semiparametric Test for Complier Characteristics}
  \author{Rahul Singh \hspace{.2cm}\\
    Department of Economics, Massachusetts Institute of Technology,\\ Cambridge, Massachusetts 02142, U.S.A.\\
    and \\
    Liyang Sun \\
    Center for Monetary and Financial Studies,\\ Madrid, 28014, Spain }
  \maketitle
} \fi

\if0\blind
{
  \bigskip
  \bigskip
  \bigskip
  \begin{center}
    {\LARGE\bf Double Robustness for Complier Parameters and a Semiparametric Test for Complier Characteristics}
\end{center}
  \medskip
} \fi
\maketitle
 
\begin{abstract}
We propose a semiparametric test to evaluate (i) whether different instruments induce subpopulations of compliers with the same observable characteristics on average, and (ii) whether compliers have observable characteristics that are the same as the full population on average. The test is a flexible robustness check for the external validity of instruments. We use it to reinterpret the difference in LATE estimates that \cite{angrist1998children} obtain when using different instrumental variables. To justify the test, we characterize the doubly robust moment for \cite{abadie2003semiparametric}’s class of complier parameters, and we analyze a machine learning update to $\kappa$ weighting. 
\end{abstract}
\textit{keywords:} Instrumental Variable; Kappa Weight; Semiparametric Efficiency.

    \section{Introduction and related work}\label{sec:intro}






Average complier characteristics help to assess the external validity of any study that uses instrumental variable identification \citep{angrist1998children,angrist2013extrapolate,swanson2013commentary,baiocchi2014instrumental,marbach2020profiling}; whose treatment effects are we estimating when we use a particular instrument? We propose a semiparametric hypothesis test, free of functional form restrictions, to evaluate (i) whether two different instruments induce subpopulations of compliers with the same observable characteristics on average, and (ii) whether compliers have observable characteristics that are the same as the full population on average. It appears that no semiparametric test previously exists for this important question about the external validity of instruments, despite the popularity of reporting average complier characteristics in empirical research, e.g. \citet[Table 2]{abdulkadirouglu2014elite}. By developing this hypothesis test, we equip empirical researchers with a new robustness check.


Equipped with this new test, we replicate, extend, and test previous findings about the impact of childbearing on female labor supply. In a seminal paper, \cite{angrist1998children} use two different instrumental variables: twin births and same-sex siblings. The two instruments give rise to two substantially different local average treatment effect (LATE) estimates for the reduction in weeks worked due to a third child: -3.28 (0.63) and -6.36 (1.18), respectively, where the standard errors are in parentheses. \cite{angrist2013extrapolate} attribute the difference in LATE estimates to a difference in average complier characteristics, i.e. a difference in average covariates for instrument specific complier
subpopulations, writing that ``twins compliers therefore are relatively more
likely to have a young second-born and to be highly educated.'' We find weak evidence in favor of the explanation that twins compliers are more
likely to have a young second-born. We do not find evidence that twins compliers have a significantly different education level than same-sex compliers.

Our test is based on a new doubly robust estimator, which we call the automatic $\kappa$ weight (Auto-$\kappa$).  To prove the validity of the test, we characterize the doubly robust moment function for average complier characteristics, which appears to have been previously unknown. More generally, we study low dimensional complier parameters that are identified using a binary instrumental variable $Z$, which is valid conditional on a possibly high dimensional vector of covariates $X$. \cite{angrist1996identification} prove that identification of LATE based on the instrumental variable does not require any functional form restrictions. Using $\kappa$ weighting, \cite{abadie2003semiparametric} extends identification for a broad class of complier parameters. As our main theoretical result, we characterize the doubly robust moment function for this class of complier parameters by augmenting $\kappa$ weighting with the classic Wald formula. Our main result answers the open question posed by \cite{sloczynski2018general} of how to characterize the doubly robust moment function for the full class, and it generalizes the well known result of \cite{tan2006regression}, who characterizes the doubly robust moment function for LATE. By characterizing the doubly robust moment function for \cite{abadie2003semiparametric}'s class of complier parameters, we handle the new and economically important case of average complier characteristics. 

The doubly robust moment function confers many favorable properties for estimation. As its name suggests, it provides double robustness to misspecification \citep{robins1995semiparametric} as well as the mixed bias property \citep{chernozhukov2018original,rotnitzky2021characterization}. As such, it allows for estimation of models in which the treatment effect for different individuals may vary flexibly according to their covariates \citep{frolich2007nonparametric,ogburn_doubly_2015}.  It also allows for nonlinear models \citep{abadie2003semiparametric,cheng2009efficient}, which are often appropriate when outcome $Y$ and treatment $D$ are binary, and therefore avoids the issue of negative weights in misspecified linear models \citep{blandhol2022tsls}. Moreover, it allows for model selection of covariates and their transformations using machine learning, as emphasized in the targeted machine learning \citep{van2006targeted,zheng2011cross,luedtke2016statistical,van2018targeted} and debiased machine learning \citep{belloni2017program,chernozhukov2016locally,chernozhukov2018original,chernozhukov2021simple} literatures. A doubly robust estimator that combines both the $\kappa$ weight and Wald formulations not only guards against misspecification but also debiases machine learning. Finally, it is semiparametrically efficient in many cases \citep{hasminskii1979nonparametric,robinson1988root,bickel1993efficient,newey1994asymptotic,robins1995semiparametric,hong2010semiparametric}. 


The structure of the paper is as follows. Section~\ref{sec:framework} defines the class of complier parameters from \cite{abadie2003semiparametric}. Section~\ref{sec:insight} summarizes our main insight: the doubly robust moment for a complier parameter combines the familiar Wald and $\kappa$ weight formulations. Section~\ref{sec:double} formalizes this insight for the full class of complier parameters. Section~\ref{sec:test} develops the practical implication of our main insight: a semiparametric test to evaluate differences in observable complier characteristics, which we use to revisit \cite{angrist1998children}. Section~\ref{sec:conc} concludes. Appendix~\ref{sec:estimation} proposes a machine learning estimator that we call the automatic $\kappa$ weight (Auto-$\kappa$), which we use to implement our proposed test.

This paper was previously circulated under a different title \citep{singh2019biased}.
\section{Framework}\label{sec:framework}

Suppose we are interested in the effect of a binary treatment $D$ on a continuous outcome $Y$ in $\mathcal{Y}$, a subset of $\mathbb{R}$. There is a binary instrumental variable $Z$ available, as well as a potentially high dimensional covariate $X$ in $\mathcal{X}$, a subset of $\mathbb{R}^{dim(X)}$. We observe $n$ independent and identically distributed observations $(W_i)$, $(i=1,...,n)$, where $W=(Y,D,Z,X^{\top})^{\top}$ concatenates the random variables. Following the notation of \cite{angrist1996identification}, we denote by $Y^{(z,d)}$ the potential outcome under the intervention $Z=z$ and $D=d$. We denote by $D^{(z)}$ the potential treatment under the intervention $Z=z$. Compliers are the subpopulation for whom $D^{(1)}>D^{(0)}$. We place standard assumptions for identification.


\begin{assumption}[Instrumental variable identification] \label{assumption:id}
Assume
\begin{enumerate}
    \item Independence: $\{Y^{(z,d)}\},\{D^{(z)}\} \indep Z\mid X$ for $d=0,1$ and $z=0,1$.
    \item Exclusion: $\text{\normalfont pr}\{Y^{(1,d)}=Y^{(0,d)}\mid X\}=1$ for $d=0,1$.
    \item Overlap: $\pi_0(X)=\text{\normalfont pr}(Z=1\mid X)$ is in $(0,1)$.
    \item Monotonicity: $\text{\normalfont pr}\{D^{(1)} \geq D^{(0)}\mid X\}=1$ and $\text{\normalfont pr}\{D^{(1)} > D^{(0)}\mid X\}>0$.
\end{enumerate}
\end{assumption}

Independence states that the instrument $Z$ is as good as randomly assigned conditional on covariates $X$. Exclusion imposes that the instrument $Z$ only affects the outcome $Y$ via the treatment $D$. We can therefore simplify notation: $Y^{(d)}=Y^{(1,d)}=Y^{(0,d)}$. Overlap ensures that there are no covariate values for which the instrument assignment is deterministic. Monotonicity rules out the possibility of defiers: individuals who will always pursue an opposite treatment status from their instrument assignment. 

\cite{angrist1996identification} prove identification of the local average treatment effect (LATE) using Assumption~\ref{assumption:id}. \cite{abadie2003semiparametric} extends identification for a broad class of complier parameters.

\begin{definition}[General class of complier parameters \citep{abadie2003semiparametric}]\label{def:class}
Let $g(y,d,x,\theta)$ be a measurable, real valued function such that $E\{g(Y,D,X ,\theta)^2\}<\infty$ for all $\theta$ in $\Theta$. Consider complier parameters $\theta_0$ implicitly defined by any of the following expressions:
\begin{enumerate}
    \item $E\{ g(Y^{(0)},X, \theta)\mid D^{(1)}>D^{(0)}\} =0$ if and only if $\theta=\theta_0$;
    \item $E\{ g(Y^{(1)},X ,\theta)\mid D^{(1)}>D^{(0)}\} =0$  if and only if $\theta=\theta_0$;
    \item $E\{g(Y,D,X,\theta)\mid D^{(1)}>D^{(0)}\}=0$  if and only if $\theta=\theta_0$.
\end{enumerate}
We subsequently refer to these expressions as the three possible cases for complier parameters.
\end{definition}

For a given instrumental variable $Z$, one may define the average complier characteristics as a special case of Definition~\ref{def:class}. This causal parameter summarizes the observable characteristics of the subpopulation of compliers who are induced to take up or refuse treatment $D$ based on the instrument assignment $Z$. It is an important parameter to estimate because it aids the interpretation of LATE. As we will see in Section~\ref{sec:test}, this causal parameter can help to reconcile different LATE estimates obtained with different instruments.
\begin{definition}[Average complier characteristics]\label{def:characteristics}
Average complier characteristics are $\theta_0=E\{f(X)\mid D^{(1)}>D^{(0)}\}$ for any measurable function $f$ of covariate $X$ that may have a finite dimensional, real vector value such that $E\{f_j(X)^2\}<\infty$.
\end{definition}

\section{Key insight}\label{sec:insight}

\subsection{Classic approaches: Wald formula and $\kappa$ weight} 

We provide intuition for our key insight that a doubly robust moment for a complier parameter has two components: the Wald formula and the $\kappa$ weight. For clarity, we focus on the familiar example of local average treatment effect (LATE) in this initial discussion: $\theta_0=E\{Y^{(1)}-Y^{(0)}\mid D^{(1)}>D^{(0)}\}$. In subsequent sections, we study the entire class of complier parameters in Definition~\ref{def:class}, including the new case of average complier characteristics.

Under Assumption~\ref{assumption:id}, LATE can be identified as 
$$
    \theta_0
=\frac{E\left\{E(Y\mid Z=1,X)-E(Y\mid Z=0,X)\right\}}{E\left\{E(D\mid Z=1,X)-E(D\mid Z=0,X)\right\}} 
$$
following \citet[Theorem 1]{frolich2007nonparametric}. We call this expression the expanded Wald formula. 

The direct Wald approach involves estimating the reduced form regression $E(Y\mid Z,X)$ and first stage regression $E(D\mid Z,X)$, then plugging these estimates into the expanded Wald formula. Such an approach is called the plug-in, and it is valid only when both regressions are estimated with correctly specified and unregularized models. It is not a valid approach when either regression is incorrectly specified, leading to the name ``forbidden regression'' \citep{angrist2008mostly}. It is also invalid when the covariates are high dimensional and a regularized machine learning estimator is used to estimate either regression. The matching procedure of \cite{frolich2007nonparametric} faces similar limitations.

In seminal work, \cite{abadie2003semiparametric} proposes an alternative formulation in terms of the $\kappa$ weights 
$$
\kappa^{(0)}(W)= (1-D)\frac{(1-Z)-\{1-\pi_0(X)\}}{\{1-\pi_0(X)\}\pi_0(X)},\quad
    \kappa^{(1)}(W)= D\frac{Z - \pi_0(X)}{\{1-\pi_0(X)\}\pi_0(X)}
$$
where $\pi_0(X)=\text{\normalfont pr}(Z=1\mid X)$ is the instrument propensity score.
The $\kappa$ weights have the property that
$$
\theta_0=\omega^{-1}E\{\kappa^{(1)}(W) Y-\kappa^{(0)}(W) Y\},\quad \omega= E\left\{1-\frac{D(1-Z)}{1-\pi_0(X)}-\frac{(1-D)Z}{\pi_0(X)}\right\}.
$$
In words, the mean of the product of $Y$ and $\kappa^{(d)}(W)$ gives, up to a scaling, the expected potential outcome $Y^{(d)}$ of compliers when treatment is $D=d$. As an aside, \cite{abadie2003semiparametric} also introduces a third weight $\kappa(W)$ for parameters that belong to the third case in Definition~\ref{def:class}.

The $\kappa$ weight approach would involve estimating the propensity score $\hat{\pi}$ and plugging this estimate into the $\kappa$ weight formula. Intuitively, the $\kappa$ weight approach is like a multistage inverse propensity weighting. Impressively, it remains agnostic about the functional form of the reduced form regression $E(Y\mid Z,X)$ and first stage regression $E(D\mid Z,X)$. It is valid only when $\hat{\pi}$ is estimated with a correctly specified and unregularized model. It is invalid if $\hat{\pi}$ is incorrectly specified or if covariates are high dimensional and a regularized machine learning estimator is used to estimate $\hat{\pi}$. Moreover, the inversion of $\hat{\pi}$ can lead to numerical instability in high dimensional settings.

\subsection{Doubly robust moment for a special case}

Next, we introduce the moment function and doubly robust moment function formulations of LATE. For the special case of LATE, these formulations were first derived by \cite{tan2006regression} with the goal of addressing misspecification of the regressions and the propensity score. Consider the expanded Wald formula. Rearranging and using the notation $V=(Y,D)^{\top}$ as a column vector,
$
\gamma_0(Z,X)=E(V\mid Z,X)$ as a vector valued regression, and $\begin{pmatrix}1, & -\theta \end{pmatrix}$ as a row vector,
we arrive at the moment function formulation of LATE:
$$
E\left[
\begin{pmatrix}1, & -\theta \end{pmatrix}
\{\gamma_0(1,X)-\gamma_0(0,X)\}
\right]
=0\text{ if and only if }\theta=\theta_0.
$$
Denote the the Horvitz-Thompson balancing weight as
$$
\alpha_0(Z,X)=\frac{Z}{\pi_0(X)}-\frac{1-Z}{1-\pi_0(X)},\quad \pi_0(X)=\text{\normalfont pr}(Z=1\mid X).
$$
\cite{tan2006regression} shows that for LATE, the doubly robust moment function is
$$
E\left[
\begin{pmatrix}1, & -\theta \end{pmatrix}
\{\gamma_0(1,X)-\gamma_0(0,X)\}
+\alpha_0(Z,X)
\begin{pmatrix}1, & -\theta \end{pmatrix}
\{V-\gamma_0(Z,X)\}\right]
=0\text{ if and only if }\theta=\theta_0.
$$
The doubly robust formulation remains valid if either the vector valued regression $\gamma_0$ or propensity score $\pi_0$ is incorrectly specified.

\subsection{A new synthesis that allows for machine learning}

Our key observation is the connection between the $\kappa$ weight and the balancing weight $\alpha_0$. This simple observation will allow us to characterize the doubly robust moment function for a broad class of complier parameters, generalizing \cite{tan2006regression} to the full class defined by \cite{abadie2003semiparametric}. 
\begin{proposition}[$\kappa$ weight as balancing weight]\label{rewrite_kappa}
The $\kappa$ weights can be rewritten as
\begin{align*}
    &\kappa^{(0)}(W)=\alpha_0(Z,X)(D-1),\quad 
    \kappa^{(1)}(W)=\alpha_0(Z,X) D,\quad 
    \kappa(W)=1-\frac{D(1-Z)}{1-\pi_{0}(X)}-\frac{(1-D)Z}{\pi_{0}(X)}.
\end{align*}
\end{proposition}

\begin{proof}
Observe that 
$$\alpha_0(z,x)=\frac{z}{\pi_0(x)}-\frac{1-z}{1-\pi_0(x)}=\frac{z-\pi_0(x)}{\pi_0(x)\{1-\pi_0(x)\}}$$
which proves the expression for $\kappa^{(0)}$ and $\kappa^{(1)}$. Using these expressions, we have 
$$\kappa(w)= \{1-\pi_0(x)\}\alpha_0(z,x)(d-1)+\pi_0(x)\alpha_0(z,x) d=1-\frac{d(1-z)}{1-\pi_{0}(x)}-\frac{(1-d)z}{\pi_{0}(x)}.$$\end{proof}

Next, we formalize the sense in which the balancing weight $\alpha_0$ represents the functional $\gamma\mapsto E\left\{\begin{pmatrix}1, & -\theta \end{pmatrix}\gamma(1,X)-\gamma(0,X)\right\}$ that appears in the moment formulation of LATE and the extended Wald formula. 

\begin{proposition}[Balancing weight as Riesz representer]\label{prop:rr}
$\alpha_0(z,x)$ is the Riesz representer to the continuous linear functional $\gamma\mapsto E\{\gamma(1,X)-\gamma(0,X)\}$, i.e. for all $\gamma$ such that $E\{\gamma(Z,X)^2\}<\infty$,
$$
E\{\gamma(1,X)-\gamma(0,X)\}=E\{\alpha_0(Z,X)\gamma(Z,X)\}.
$$
Similarly, $Z/\pi_0(X)$ is the Riesz representer to the continuous linear functional $\gamma\mapsto E\{\gamma(1,X)\}$, and $(1-Z)/\{1-\pi_0(X)\}$ is the Riesz representer to the continuous linear functional $\gamma\mapsto E\{\gamma(0,X)\}$.
\end{proposition}

\begin{proof}
This result is well known in semiparametrics. We provide the proof for completeness. Observe that
\begin{align*}
     E\left\{\gamma(Z,X) \frac{Z}{\pi_0(X)}\mid X\right\} 
     &= E\left\{\gamma(Z,X) \frac{1}{\pi_0(X)}\mid Z=1,X\right\}\text{\normalfont pr}(Z=1\mid X) \\
     &= E\left\{\gamma(Z,X) \frac{1}{\pi_0(X)}\mid Z=1,X\right\}\pi_0(X)
     =\gamma(1,X)
 \end{align*}
 and likewise
  $$
 E\left\{\gamma(Z,X) \frac{1-Z}{1-\pi_0(X)}\mid X\right\}=\gamma(0,X).
 $$ 
 Combining these two terms, we have by the law of iterated expectations
 \begin{align*}
     &E\{\gamma(1,X)-\gamma(0,X)\}
     = \int \{\gamma(1,x)-\gamma(0,x)\} \mathrm{d}\text{\normalfont pr}(x) \\
     &=\int \left[E\left\{\gamma(Z,X) \frac{Z}{\pi_0(X)}\mid X=x\right\}-E\left\{\gamma(Z,X) \frac{1-Z}{1-\pi_0(X)}\mid X=x\right\}\right] \mathrm{d}\text{\normalfont pr}(x) \\
     &=E\left\{\gamma(Z,X)\frac{Z}{\pi_0(X)}\right\}-E\left\{\gamma(Z,X)\frac{1-Z}{1-\pi_0(X)}\right\}.
 \end{align*}
\end{proof}
An immediate consequence of Proposition~\ref{prop:rr} is that
$$
    E\left\{\begin{pmatrix}1, & -\theta \end{pmatrix}\gamma(1,X)-\gamma(0,X)\right\}=E\left\{\alpha_0(Z,X)\begin{pmatrix}1, & -\theta \end{pmatrix}\gamma(Z,X)\right\} \text{ for any } \gamma.
$$

In summary, Proposition~\ref{rewrite_kappa} shows that the $\kappa$ weight is a reparametrization of the balancing weight $\alpha_0$. Meanwhile, Proposition~\ref{prop:rr} shows that the balancing weight appears in the Riesz representer to the moment formulation of LATE, i.e. the expanded Wald formula. We conclude that the $\kappa$ weight is essentially the Riesz representer to the Wald formula.  In seminal work, \cite{newey1994asymptotic} demonstrates that a doubly robust moment is constructed from a moment formulation and its Riesz representer.  Therefore the doubly robust moment for complier parameters must combine the Wald formula and the $\kappa$ weight.

With the general doubly robust moment function, one can propose flexible, semiparametric tests for complier parameters. In particular, the semiparametric tests may involve regularized machine learning for flexible estimation and model selection of (i) the regression $\hat{\gamma}$ in a way that approximates nonlinearity and heterogeneity, and (ii) the balancing weight $\hat{\alpha}$ in a way that guarantees balance. In Section~\ref{sec:test}, we instantiate such a test to compare observable characteristics of compliers. 

As explained in Appendix~\ref{sec:estimation}, we avoid the numerically unstable step of estimating and inverting $\hat{\pi}$ that appears in \cite{tan2006regression,belloni2017program,chernozhukov2018original}. We replace it with the numerically stable step of estimating $\hat{\alpha}$ directly, extending techniques of \cite{chernozhukov2018learning} to the instrumental variable setting. We call this extension automatic $\kappa$ weighting (Auto-$\kappa$), and demonstrate how it applies to the new and economically important case of average complier characteristics.

In summary, our main theoretical result allows us to combine the classic Wald and $\kappa$ weight formulations for the entire class of complier parameters in Definition~\ref{def:class}, including average complier characteristics, while also updating them to incorporate machine learning.

\section{The doubly robust moment}\label{sec:double}

We now state our main theoretical result, which is the doubly robust moment for the class of complier parameters in Definition~\ref{def:class}. This result formalizes the intuition of Section~\ref{sec:insight}, and it justifies the hypothesis test in Section~\ref{sec:test}. It is convenient to divide the main result into two statements for clarity. Theorem~\ref{thm:general_kappa} handles the first and second cases in Definition~\ref{def:class}, while Theorem~\ref{thm:general_kappa3} handles the third case in Definition~\ref{def:class}.

\begin{theorem}[Cases 1 and 2]\label{thm:general_kappa}
Suppose Assumption~\ref{assumption:id} holds. Let $g(y,d,x,\theta)$ be a measurable, real valued function such that $E\{g(Y,D,X ,\theta)^2\}<\infty$ for all $\theta$ in $\Theta$. 
\begin{enumerate}
    \item If $\theta_0$ is defined by $E[ g\{Y^{(0)},X ,\theta_0\}\mid D^{(1)}>D^{(0)}] =0$, let 
    $v(w,\theta)=(d-1) g(y,x,\theta).$
    \item If $\theta_0$ is defined by $E[ g\{Y^{(1)},X ,\theta_0\}\mid D^{(1)}>D^{(0)}] =0$, let 
    $v(w,\theta)=d  g(y,x,\theta).$
\end{enumerate}
Then the doubly robust moment function $\psi$ for $\theta_0$ is of the form
\begin{align*}
    &\psi(w,\gamma,\alpha,\theta)=m(w,\gamma,\theta)+\phi(w,\gamma,\alpha,\theta),\quad 
    m(w,\gamma,\theta)=\gamma(1,x,\theta)-\gamma(0,x,\theta),\\
    &\phi(w,\gamma,\alpha,\theta)=\alpha(z,x)\{v(w,\theta)-\gamma(z,x, \theta)\}
\end{align*}  
where $\gamma_0(z,x,\theta)=E\{v(W,\theta)\mid z,x\}$ is a vector valued regression and $\alpha_0(z,x)=z/\pi_0(x)-(1-z)/\{1-\pi_0(x)\}$ is the Riesz representer of the functional $\gamma\mapsto E\{\gamma(1,X,\theta)-\gamma(0,X,\theta)\}$.
\end{theorem}

\begin{proof}
Consider the first case. Under Assumption~\ref{assumption:id}, we can appeal to \citet[Theorem 3.1]{abadie2003semiparametric}:
    $$
    0=E[ g\{Y^{(0)},X ,\theta_0\}\mid D^{(1)}>D^{(0)}]=\dfrac{E\{\kappa ^{(0)}(W)g(Y,X ,\theta_0)\}}{\text{\normalfont pr}\{D^{(1)}>D^{(0)}\}}.
    $$
    Hence 
  \begin{align*}
      0&=E\{\kappa ^{(0)}(W)g(Y,X ,\theta_0)\} 
      =E\{\alpha_0(Z,X)(D-1)g(Y,X ,\theta_0)\}
      =E\{\alpha_0(Z,X)v(W,\theta_0)\} \\
      &=E\{\alpha_0(Z,X)\gamma_0(Z,X,\theta_0)\} 
      =E\{\gamma_0(1,X,\theta_0)-\gamma_0(0,X,\theta_0)\}
  \end{align*}
appealing to the previous statement, Proposition~\ref{rewrite_kappa}, the definition of $v(W,\theta_0)$, the law of iterated expectations, and Proposition~\ref{prop:rr}. Likewise for the second case.
\end{proof}

In the doubly robust moment function $\psi(w,\gamma,\alpha,\theta)=m(w,\gamma,\theta)+\phi(w,\gamma,\alpha,\theta)$, we generalize our insight from Section~\ref{sec:insight}. The first term $m(w,\gamma,\theta)$ is essentially a generalized Wald formula. The second term $\phi(w,\gamma,\alpha,\theta)$ is essentially a product between the $\kappa$ weight and a generalized regression residual. In the language of semiparametrics, we \textit{augment} the $\kappa$ weight with the Wald formula. Equivalently, we \textit{debias} the Wald formula with the $\kappa$ weight.

The doubly robust moment function $\psi$ remains valid if either $\gamma_0$ or $\alpha_0$ is misspecified, i.e.
$$
0=E\{\psi(W,\gamma,\alpha_0,\theta_0)=E[\psi(W,\gamma_0,\alpha,\theta_0)\} \text{ for any } \gamma,\alpha.
$$
In the former expression, $\gamma_0$ may be misspecified yet $\psi$ remains valid as an estimating equation. In the latter, $\alpha_0$ may be misspecified yet $\psi$ remains valid as an estimating equation. Theorem~\ref{thm:general_kappa} demonstrates that all complier parameters in cases 1 and 2 of Definition~\ref{def:class} have a doubly robust moment function $\psi$ with a common structure. As such, we are able to analyze all of these causal parameters with the same argument. Case 3 of Definition~\ref{def:class} is more involved, but we show that it shares the common structure as well.

\begin{theorem}[Case 3]\label{thm:general_kappa3}
Suppose Assumption~\ref{assumption:id} holds. Let $g(y,d,x,\theta)$ be a measurable, real valued function such that $E\{g(Y,D,X ,\theta)^2\}<\infty$ for all $\theta$ in $\Theta$. 
If $\theta_{0}$ is defined by the moment condition $E\{g(Y,D,X,\theta_{0})\mid D^{(1)}>D^{(0)}\}=0$, 
     then the doubly robust moment function for $\theta_{0}$ is
of the form 
\begin{align*}
&\psi(w,\tilde{\gamma},\tilde{\alpha},\theta) =m(w,\tilde{\gamma},\theta)+\phi(w,\tilde{\gamma},\tilde{\alpha},\theta),\quad 
m(w,\tilde{\gamma},\theta) =\gamma(z,x,\theta)-\gamma^{0}(1,x,\theta)-\gamma^{1}(0,x,\theta)\\
&\phi(w,\tilde{\gamma},\tilde{\alpha},\theta) =\{g(y,d,x,\theta)-\gamma(z,x,\theta)\} -\alpha^0(z,x)\{(1-d) g(y,d,x,\theta)-\gamma^{0}(z,x,\theta)\}\\
 &\quad\quad\quad\quad\quad\quad -\alpha^1(z,x)\{d g(y,d,x,\theta)-\gamma^{1}(z,x,\theta)\}
\end{align*}
 where $\tilde{\gamma}$ concatenates $(\gamma,\gamma^0,\gamma^1)$ and $\tilde{\alpha}$ concatenates $(\alpha^0,\alpha^1)$. These functions are defined by
 \begin{align*}
     &\gamma_{0}(z,x,\theta)=E\{g(Y,D,X,\theta)\mid z,x\},\quad \gamma_{0}^{0}(z,x,\theta)=E\{(1-D) g(Y,D,X,\theta)\mid z,x\},\\ &\gamma_{0}^{1}(z,x,\theta)=E\{D g(Y,D,X,\theta)\mid z,x\},\quad \alpha_0^0(z,x)=z/\pi_0(x),\quad  \alpha_0^1(z,x)=(1-z)/\{1-\pi_0(x)\}.
 \end{align*}
\end{theorem}

\begin{proof}
A similar argument extends to the third case. Under Assumption~\ref{assumption:id},
we can appeal to \citet[Theorem 3.1]{abadie2003semiparametric}:
$$
0=E\{g(Y,D,X,\theta_{0})\mid D^{(1)}>D^{(0)}\}=\dfrac{E\{\kappa(W)g(Y,D,X,\theta_{0})\}}{\text{\normalfont pr}\{D^{(1)}>D^{(0)}\}}.
$$
 Hence 
\begin{align*}
0 & =E\{\kappa(W)g(Y,D,X,\theta_{0})\}\\
 & =
 E\left\{g(Y,D,X,\theta_{0})
 -\frac{Z}{\pi_{0}(X)}(1-D) g(Y,D,X,\theta_{0})
 -\frac{1-Z}{1-\pi_{0}(X)}D g(Y,D,X,\theta_{0})\right\}\\
  & =E\left\{\gamma_{0}(Z,X,\theta_0)-\frac{Z}{\pi_{0}(X)}\gamma_{0}^{0}(Z,X,\theta_{0})-\frac{1-Z}{1-\pi_{0}(X)}\gamma_{0}^{1}(Z,X,\theta_{0})\right\}\\
 & =E\{\gamma_{0}(Z,X,\theta_0)-\gamma_{0}^{0}(1,X,\theta_{0})-\gamma_{0}^{1}(0,X,\theta_{0})\}
\end{align*}
 appealing to the previous statement, Proposition~\ref{rewrite_kappa}, the definitions of $(\gamma_0,\gamma_0^0,\gamma_0^1)$ together with the law of iterated expectations, and Proposition~\ref{prop:rr}.
\end{proof}

This time, the doubly robust moment function $\psi$ remains valid if either $\tilde{\gamma}_0$ or $\tilde{\alpha}_0$ is misspecified, i.e.
$$
0=E\{\psi(W,\tilde{\gamma},\tilde{\alpha}_0,\theta_0)=E[\psi(W,\tilde{\gamma}_0,\tilde{\alpha},\theta_0)\} \text{ for any } \tilde{\gamma},\tilde{\alpha}.
$$
In the former expression, $\tilde{\gamma}_0$ may be misspecified yet $\psi$ remains valid as an estimating equation. In the latter, $\tilde{\alpha}_0$ may be misspecified yet $\psi$ remains valid as an estimating equation. 

In Section~\ref{sec:test}, we translate this general characterization of the doubly robust moment into a practical hypothesis test to evaluate the external validity of instruments. In Appendix~\ref{sec:estimation}, we translate this general characterization into general machine learning estimators for complier parameters, which we use to implement the hypothesis test. In particular, we consider direct estimation of the balancing weight, a procedure that we call automatic $\kappa$ weighting (Auto-$\kappa$).
\section{A hypothesis test to compare observable characteristics}\label{sec:test}

\subsection{Corollaries for average complier characteristics}

As a corollary, we characterize the doubly robust moment for average complier characteristics, which appears to have been previously unknown. Using the new doubly robust moment, we propose a hypothesis test, free of functional form restrictions, to evaluate (i) whether two different instruments induce subpopulations of compliers with the same observable characteristics on average, and (ii) whether compliers have observable characteristics that are the same as the full population on average.

\begin{corollary}[Average complier characteristics]\label{cor:characteristics}
The doubly robust moment for average complier characteristics is
$$
 \psi(w,\gamma,\alpha,\theta)=A(\theta)\{\gamma(1,x)-\gamma(0,x)\}+\alpha(z,x)A(\theta)\{v-\gamma(z,x)\},\quad A(\theta)=\begin{pmatrix}I, & -\theta \end{pmatrix}
$$
where
$
v=\{df(x)^{\top}, d\}^{\top}$, $\gamma_0(z,x)=E(V\mid z,x)$, and $\alpha_0(z,x)=z/\pi_0(x)-(1-z)/\{1-\pi_0(x)\}$.
\end{corollary}

\begin{proof}
The result is a special case of Corollary~\ref{cor:LATE} in Appendix~\ref{sec:estimation}.
\end{proof}

Suppose we wish to test the null hypothesis that two different instruments $Z_1$ and $Z_2$ induce complier subpopulations with the same observable characteristics on average. Denote by $\hat{\theta}_1$ and $\hat{\theta}_2$ the estimators for average complier characteristics using the different instruments $Z_1$ and $Z_2$, respectively. One may construct machine learning estimators $\hat{\theta}_1$ and $\hat{\theta}_2$ based on the doubly robust moment function in Corollary~\ref{cor:characteristics}. In Appendix~\ref{sec:estimation}, we instantiate automatic $\kappa$ weight (Auto-$\kappa$) estimators of this type. The following procedure allows us to test the null hypothesis from some estimator $\hat{C}$ for the asymptotic variance $C$ of $\hat{\theta}=(\hat{\theta}^{\top}_1,\hat{\theta}^{\top}_2)^{\top}$. In Appendix~\ref{sec:estimation}, we provide an explicit variance estimator $\hat{C}$ based on Auto-$\kappa$ as well.

\begin{algorithm}[Hypothesis test for difference of average complier characteristics]\label{alg:hyp}
Given $\hat{\theta}$ and $\hat{C}$, which may be based on Auto-$\kappa$ as in Appendix~\ref{sec:estimation},
\begin{enumerate}
    \item Calculate the statistic $T=n(\hat{\theta}_1-\hat{\theta}_2)^{\top}(R
    \hat{C}
    R^{\top})^{-1}(\hat{\theta}_1-\hat{\theta}_2)$ where $ R=\begin{pmatrix}I, & -I\end{pmatrix}$.
    \item Compute the value $c_{a}$ as the $(1-a)$ quantile of $\chi^2\{dim(\theta_1)\}$.
    \item Reject the null hypothesis if $T>c_{a}$.
\end{enumerate}
\end{algorithm}

Algorithm~\ref{alg:hyp} can also test the null hypothesis that compliers have observable characteristics that are the same as the full population on average. $\hat{\theta}_1$ is as before, $\hat{\theta}_2=n^{-1}\sum_{i=1}^nf(X_i)$, and $\hat{C}$ updates accordingly.

\begin{corollary}[Hypothesis test for difference of average complier characteristics]\label{cor:hyp}
If $\hat{\theta}=\theta_0+o_p(1)$,  $n^{1/2}(\hat{\theta}-\theta_0)\rightsquigarrow \mathcal{N}(0,C)$, and $\hat{C}=C+o_p(1)$, then the hypothesis test in Algorithm~\ref{alg:hyp} 
falsely rejects the null hypothesis $H_0$ with
probability approaching the nominal level, i.e.
$
\text{\normalfont pr}(T>c_{a}\mid H_0)\rightarrow a.
$
\end{corollary}

\begin{proof}
The result is immediate from \citet[Section 9]{newey1994large}.
\end{proof}

Corollary~\ref{cor:hyp} is our main practical result: justification of a flexible hypothesis test to evaluate a difference in average complier characteristics. It appears that no semiparametric test previously exists for this important question about the external validity of instruments. By developing this hypothesis test, we equip empirical researchers with a new robustness check. This practical result follows as a consequence of our main insight in Section~\ref{sec:insight} and our main theoretical result in Section~\ref{sec:double}. In Appendix~\ref{sec:estimation}, we verify the conditions of Corollary~\ref{cor:hyp} for Auto-$\kappa$ under weak regularity assumptions.

\subsection{Empirical application}

With this practical result, we revisit a classic empirical paper in labor economics to test whether two different instruments induce different average complier characteristics. \cite{angrist1998children} estimate the impact of childbearing $D$ on female labor supply $Y$ in a sample of 394,840 mothers, aged 21--35 with at least two children, from the 1980 Census. The first instrument $Z_1$ is twin births: $Z_1$ indicates whether the mother's second and third children were twins. The second instrument $Z_2$ is same-sex siblings: $Z_2$ indicates whether the mother's initial two children were siblings with the same sex. The authors reason that both $(Z_1,Z_2)$ are quasi random events that induce having a third child.

\begin{table}[h]
\caption{Comparison of average complier characteristics}{%
\begin{tabular}{lcccccccc}
 \\
 & \multicolumn{4}{c}{Average age of second child} & \multicolumn{4}{c}{Average schooling of mother} \\  
 & Twins & Same-sex &  2 sided & 1 sided & Twins & Same-sex &  2 sided & 1 sided\\[5pt]
$\kappa$ weight & 5.51 & 7.14 & - & - & 12.43 & 12.09 & - & - \\
Auto-$\kappa$ & 4.52 & 6.92 & 0.13 & 0.07 & 9.84 & 12.10 & 0.54 & 0.27 \\
Auto-$\kappa $ (S.E.) & (0.70) & (1.43) & - & - & (2.47) & (2.78) & - & -
\end{tabular}}
\label{tab:1}
\textit{Notes:}
S.E., standard error; Auto-$\kappa$, automatic $\kappa$ weighting. See Supplement~\ref{sec:application} for estimation details.
\end{table}

The two instruments give rise to two LATE estimates for the reduction in weeks worked due to a third child: -3.28 (0.63) for $Z_1$ and -6.36 (1.18) for $Z_2$, where the standard errors are in parentheses. \cite{angrist2013extrapolate} attribute the difference in LATE estimates to a difference in average complier characteristics, i.e. a difference in average covariates for instrument specific complier
subpopulations. The authors use parametric $\kappa$ weights, report point estimates without standard errors, and conclude that
``twins compliers therefore are relatively more
likely to have a young second-born and to be highly educated.'' 

We replicate, extend, and test these previous findings. In their parametric $\kappa$ weight approach, \cite{angrist2013extrapolate} estimate $\pi_{0}(X)$ using a logistic model with polynomials of continuous covariates. In our semiparametric Auto-$\kappa$ approach, we expand the dictionary to higher order polynomials, include interactions between the instrument and covariates, and directly estimate and regularize the balancing weights. Crucially, our main result allows us to conduct inference, and to test whether the instruments $Z_1$ and $Z_2$ induce differences in the observable complier characteristics suggested by previous work.

Table~\ref{tab:1} summarizes results. In Columns 1, 2, 5, and 6, we find similar point estimates to \cite{angrist2013extrapolate}, given in Row 1. Columns 3, 4, 7, and 8 report $p$ values for tests of the null hypothesis that average complier characteristics are equal for the twins and same-sex instruments. We find weak evidence in favor of the explanation that twins compliers are more
likely to have a young second-born. We do not find evidence that twins compliers have a significantly different education level than same-sex compliers.

\section{Conclusion}\label{sec:conc}

We propose a semiparametric test to evaluate (i) whether two different instruments induce subpopulations of compliers with the same observable characteristics on average, and (ii) whether compliers have observable characteristics that are the same as the full population on average. This hypothesis test is a flexible and practical robustness check for the external validity of instrumental variables. We use the test to reinterpret the difference in LATE estimates that \cite{angrist1998children} obtain when using two different instrumental variables. Specifically, we implement a machine learning update to $\kappa$ weighting that we call the automatic $\kappa$ weight (Auto-$\kappa$). To justify the test, we develop new econometric theory. Most notably, we characterize the doubly robust moment function for the entire class of complier parameters from \cite{abadie2003semiparametric}, answering an open question in the semiparametric literature in order to handle the new and economically important case of average complier characteristics.

\appendix

    \section{Automatic $\kappa$ weights}\label{sec:estimation}

\subsection{Estimation}

In Section~\ref{sec:double}, we present our main theoretical result: the doubly robust moment function for the class of complier parameters in Definition~\ref{def:class}. In this section, we propose a machine learning estimator based on this doubly robust moment function, which we call automatic $\kappa$ weighting (Auto-$\kappa$). We verify the conditions of Corollary~\ref{cor:hyp} using Auto-$\kappa$. In doing so, we provide a concrete end-to-end procedure to test whether two different instruments induce subpopulations of compliers with the same observable characteristics.

Debiased machine learning \citep{chernozhukov2016locally,chernozhukov2018original} is a meta estimation procedure that combines doubly robust moment functions \citep{robins1995semiparametric} with sample splitting \citep{klaassen1987consistent}. Given the doubly robust moment function of some causal parameter of interest as well as machine learning estimators $(\hat{\gamma},\hat{\alpha})$ for its nonparametric components, debiased machine learning generates an estimator of the causal parameter. 
\begin{algorithm}[Debiased machine learning]\label{alg_dml}
Partition the sample into subsets $(I_{\ell})$, $(\ell=1,...,L)$. 
\begin{enumerate}
    \item For each $\ell$, estimate $\hat{\gamma}_{-\ell}$ and $\hat{\alpha}_{-\ell}$ from observations not in $I_{\ell}$.
    \item Estimate $\hat{\theta}$ as the solution to 
    $
   n^{-1}\sum_{\ell=1}^L\sum_{i \in I_{\ell}} \psi(W_i,\hat{\gamma}_{-\ell},\hat{\alpha}_{-\ell},\theta)|_{\theta=\hat{\theta}}=0.
    $
\end{enumerate}
\end{algorithm}
In Theorems~\ref{thm:general_kappa} and~\ref{thm:general_kappa3}, we characterize the doubly robust moment function $\psi$ for complier parameters. What remains is an account of how to estimate the vector valued regression $\hat{\gamma}$ and the balancing weight $\hat{\alpha}$. Our theoretical results are agnostic about the choice of $(\hat{\gamma},\hat{\alpha})$ as long as they satisfy the rate conditions in Assumption~\ref{kappa_regularity} below. For example, $\hat{\gamma}$ could be a neural network. 

For the balancing weight estimator $\hat{\alpha}$, we adapt the regularized Riesz representer of \cite{chernozhukov2018learning}, though one could similarly adapt the minimax balancing weight of \cite{hirshberg2021augmented}. This aspect of the procedure departs from the explicit inversion of the propensity score in \cite{tan2006regression,belloni2017program,chernozhukov2018original}, and it improves numerical stability, which we demonstrate though comparative simulations in Supplement~\ref{sec:sim}. In particular, we project the balancing weight $\alpha_0(Z,X)$ onto the $p$ dimensional dictionary of basis functions $b(Z,X)$. A high dimensional dictionary allows for flexible approximation, which we discipline with $\ell_1$ regularization. 

\begin{algorithm}[Regularized balancing weight]\label{alg_stage1}
Based on the observations in $I_{-\ell}$,
\begin{enumerate}
    \item Calculate $p\times p$ matrix $\hat{G}_{-\ell}=(n-n_{\ell})^{-1}\sum_{i\in I_{-\ell}} b(Z_i,X_i)b(Z_i,X_i)^{\top}$,
    \item Calculate $p\times 1$ vector $\hat{M}_{-\ell}=(n-n_{\ell})^{-1}\sum_{i\in I_{-\ell}} b(1,X_i)-b(0,X_i)$,
    \item Set $\hat{\alpha}_{-\ell}(Z,X)=b(Z,X)^{\top}\hat{\rho}_{-\ell}$ where
    $
    \hat{\rho}_{-\ell}=\argmin_{\rho} \rho^{\top}\hat{G}_{-\ell}\rho-2\rho^{\top}\hat{M}_{-\ell}+2\lambda_n|\rho|_1.
    $
\end{enumerate}
\end{algorithm}

We refer to our proposed estimator, which combines the doubly robust moment function from Theorems~\ref{thm:general_kappa} and~\ref{thm:general_kappa3} with the meta procedure in Algorithm~\ref{alg_dml} and the regularized balancing weights in Algorithm~\ref{alg_stage1}, as automatic $\kappa$ weighting (Auto-$\kappa$) for complier parameters. The new doubly robust moment in Corollary~\ref{cor:characteristics} means that Auto-$\kappa$ applies to the new and economically important case of average complier characteristics.

\subsection{Affine moments}

When we verify the conditions of Corollary~\ref{cor:hyp} using Auto-$\kappa$, we focus on a sub-class of the complier parameters in Definition~\ref{def:class}. This sub-class is rich enough to include several empirically important parameters, yet simple enough to avoid iterative estimation. The sub-class consists of complier parameters with affine moments, which we now define. The affine moment condition can be relaxed, but doing so incurs iterative estimation \citep{chernozhukov2016locally}.

\begin{definition}[Affine moment]\label{def:affine}
We say a doubly robust moment function $\psi$ is affine in $\theta$ if it takes the form
$$
 \psi(W,\gamma,\alpha,\theta)=A(\theta)\{\gamma(1,X)-\gamma(0,X)\}+\alpha(Z,X)A(\theta)\{V-\gamma(Z,X)\}
$$
where $A(\theta)$ is a matrix with entries that are ones, zeros, or components of $\theta$.
\end{definition}

Next, we verify that several empirically important complier parameters have affine moments.

\begin{definition}[Empirically important complier parameters]\label{def:complier-pop} Consider the following popular parameters.
\begin{enumerate}
    \item LATE is $\theta_0=E\{Y^{(1)}-Y^{(0)}\mid  D^{(1)}>D^{(0)}\}$.
    \item Average complier characteristics are $\theta_0=E\{f(X)\mid D^{(1)}>D^{(0)}\}$ for any measurable function $f$ of covariate $X$ that may have a finite dimensional, real vector value such that $E\{f_j(X)^2\}<\infty$.
    \item Complier counterfactual outcome distributions are $\theta_0=(\theta_0^y)_{y\in\mathcal{U}}$ where $$\theta_0^y=\begin{pmatrix} \beta_0^y\\\delta_0^y\end{pmatrix}
    =\begin{bmatrix} \text{\normalfont pr}\{Y^{(0)}\leq y\mid D^{(1)}>D^{(0)}\}\\\text{\normalfont pr}\{Y^{(1)}\leq y\mid D^{(1)}>D^{(0)}\}\end{bmatrix}$$ and $\mathcal{U} \subset \mathcal{Y}$  is a fixed grid of finite dimension.
\end{enumerate}
\end{definition}

\begin{corollary}[Empirically important parameters have affine moments]\label{cor:LATE}
Under Assumption~\ref{assumption:id}, the doubly robust moment functions for LATE, average complier characteristics, and complier counterfactual outcome distributions are affine, where
\begin{enumerate}
    \item For LATE \citep{tan2006regression}, we set $V=(Y, D)^{\top}$ and $ A(\theta)=\begin{pmatrix}1, & -\theta \end{pmatrix}
$.
    \item For complier characteristics, we set $
V=(Df(X)^{\top}, D)^{\top}$ and $ A(\theta)=\begin{pmatrix}I, & -\theta \end{pmatrix}.
$
\item For complier counterfactual distributions \citep{belloni2017program}, we set 
$$
V^y=\{(D-1)1_{Y\leq y},D 1_{Y\leq y} , D\}^{\top}\text{ and } A(\theta^y)=\begin{pmatrix}1 & 0 & -\beta^y  \\ 0&1& -\delta^y \end{pmatrix}.
$$
\end{enumerate}
\end{corollary}

\begin{proof}
Suppose we can decompose $v(w,\theta)=h(w,\theta) + a(\theta)$ for some function $a(\cdot)$ that does not depend on data. Then we can replace $v(w,\theta)$ with  $h(w,\theta)$ without changing $m$ and $\phi$ in the sense of Theorem~\ref{thm:general_kappa}.  This is because $$E\{v(W,\theta) \mid z,x\} = E\{h(W,\theta)\mid z,x\} + a(\theta) $$ and hence  $$v(w,\theta)-E\{v(W,\theta) \mid z,x\}=h(w,\theta)-E\{h(W,\theta)\mid z,x\}.$$ Whenever we use this reasoning, we write $v(w,\theta)\propto h(w,\theta)$.
\begin{enumerate}
    \item For LATE we can write $\theta_0 = \delta_0 - \beta_0$, where $\delta_0$ is defined by the moment condition $E\{ Y^{(1)}-\delta_0\mid D^{(1)}>D^{(0)}\} =0$ and $\beta_0$ is defined by the moment condition $E\{ Y^{(0)}-\beta_0\mid D^{(1)}>D^{(0)}\} =0$. Applying Case 2 of Theorem~\ref{thm:general_kappa} to $\delta_0$, we have $v(w,\delta)=d (y - \delta)$. Applying Case 1 of Theorem~\ref{thm:general_kappa} to $\beta_0$, we have $v(w,\beta)= (d-1) (y - \beta)\propto (d-1) y - d \beta$. Writing $\theta=\delta-\beta$, the moment function for $\theta_0$ can thus be derived with $$v(w,\theta)=v(w,\delta)-v(w,\beta) = y - d \theta.$$ 
    This expression decomposes into $V=(Y, D)^{\top}$ and $ A(\theta)=\begin{pmatrix}1, & -\theta \end{pmatrix}$ in Corollary~\ref{cor:LATE}.
    
    \item For average complier characteristics,  $\theta_0$ is defined by the moment condition $E\{ f(X)-\theta_0\mid D^{(1)}>D^{(0)}\}=0$.  Applying Case 2 of Theorem~\ref{thm:general_kappa} setting $g(Y^{(1)},X,\theta_0) = f(X)-\theta_0$, we have $v(w,\theta)=d(f(x)-\theta)$. This expression decomposes into $V=(Df(X)^{\top}, D)^{\top}$ and $ A(\theta)=\begin{pmatrix}I, & -\theta \end{pmatrix}$  in Corollary~\ref{cor:LATE}.
    
    \item For complier distribution of $Y^{(0)}$, $\beta^{\bar{y}}_0$ is defined by the moment condition $E\{ 1_{Y^{(0)}\leq \bar{y}}-\beta^{\bar{y}}_0\mid D^{(1)}>D^{(0)}\} =0$. Applying Case 1 of Theorem~\ref{thm:general_kappa} to $\beta^{\bar{y}}_0$, we have $v(w,\beta^{\bar{y}})= (d-1) ( 1_{y \leq {\bar{y}}} - \beta^{\bar{y}}) \propto (d-1) 1_{y \leq {\bar{y}}} - d \beta^{\bar{y}}$. For complier distribution of $Y^{(1)}$, $\delta^{\bar{y}}_0$ is defined by the moment condition $E\{ 1_{Y^{(1)}\leq {\bar{y}}}-\delta^{\bar{y}}_0\mid D^{(1)}>D^{(0)}\} =0$. Applying Case 2 of Theorem~\ref{thm:general_kappa} to $\delta_0$, we have $v(w,\delta^{\bar{y}})= d ( 1_{y \leq {\bar{y}}} - \delta^{\bar{y}})$.  Concatenating $v(w,\beta^{\bar{y}})$ and $v(w,\delta^{\bar{y}})$, we arrive at the decomposition in  Corollary~\ref{cor:LATE}.
\end{enumerate}

\end{proof}

\subsection{Inference}

We prove the Auto-$\kappa$ estimator for complier parameters is consistent, asymptotically normal, and semiparametrically efficient. In doing so, we verify the conditions of Corollary~\ref{cor:hyp}. We build on the theoretical foundations in \cite{chernozhukov2016locally} to generalize the main result in \cite{chernozhukov2018learning}. We assume the following regularity conditions.

\begin{assumption}[Regularity conditions for complier parameter estimation]\label{kappa_regularity}
Assume
\begin{enumerate}
\item Affine moment: $\psi$ is affine in $\theta$;
    \item Bounded propensity: $\pi_0(X)$ is in $(\bar{c},1-\bar{c})$ for some $\bar{c}>0$ uniformly over the support of $X$;
    \item Bounded variance: $\text{\normalfont var}(V\mid Z,X)$ is bounded uniformly over the support of $(Z,X)$;
    \item Nonsingular Jacobian: $J=E\left\{\partial \psi(W,\gamma_0,\alpha_0,\theta) / \partial\theta  |_{\theta=\theta_0}\right\}$ is nonsingular;
    \item Compact parameter space: $\theta_0,\hat{\theta}$ are in $\Theta$, a compact parameter space;
    \item Rates: $|\hat{\alpha}|_{\infty}=O_p(1)$, $\|\hat{\alpha}-\alpha_0\|=o_p(1)$, $\|\hat{\gamma}-\gamma_0\|=o_p(1)$, and $\|\hat{\alpha}-\alpha_0\| \|\hat{\gamma}-\gamma_0\|=o_p(n^{-1/2})$.
\end{enumerate}
\end{assumption}

The most substantial condition in Assumption~\ref{kappa_regularity} is the rate condition, where we use the notation $\|V_j\|=\{E(V_j^2)\}^{1/2}$ and $\|V\|=\{\|V_1\|,...,\|V_{dim(V)}\|\}^{\top}$. In Supplement~\ref{sec:rate}, we verify the rate condition for the $\hat{\alpha}$ estimator in Algorithm~\ref{alg_stage1}. Since $\hat{\gamma}$ is a standard nonparametric regression, a broad variety of estimators and their mean square rates can be quoted to satisfy the rate condition for $\hat{\gamma}$. The product condition formalizes the mixed bias property. It allows \textit{either} the convergence rate of $\hat{\gamma}$ to be slower than $n^{-1/4}$ \textit{or} the convergence rate of $\hat{\alpha}$ to be slower than $n^{-1/4}$, as long as the other convergence rate is faster than $n^{-1/4}$. As such, it allows \textit{either} $\hat{\gamma}$ to be a complicated function \textit{or} $\hat{\alpha}$ to be a complicated function, as long as the other is a simple function, in a sense that we formalize in Supplement~\ref{sec:rate}.

\begin{theorem}[Consistency and asymptotic normality]\label{asymptotic_normality}
Suppose Assumption~\ref{kappa_regularity} holds. Then $\hat{\theta}=\theta_0+o_p(1)$,  $n^{1/2}(\hat{\theta}-\theta_0)\rightsquigarrow \mathcal{N}(0,C)$, and $\hat{C}=C+o_p(1)$ where
\begin{align*}
    &J=E\left\{\frac{\partial \psi_0(W)}{\partial\theta}\right\},\;
    \hat{J}=\frac{1}{n}\sum_{\ell=1}^L\sum_{i \in I_{\ell}}\frac{\partial \hat{\psi}_i(\hat{\theta})}{\partial \theta},\;
    \Omega=E\{\psi_0(W)\psi_0(W)^{\top}\},\;
    \hat{\Omega}=\frac{1}{n}\sum_{\ell=1}^L\sum_{i \in I_{\ell}} \hat{\psi}_i(\hat{\theta})\hat{\psi}_i(\hat{\theta})^{\top} \\
    &C=J^{-1}\Omega J^{-1},\quad
    \hat{C}=\hat{J}^{-1}\hat{\Omega} \hat{J}^{-1},\quad 
    \psi_0(W)=\psi(W,\gamma_0,\alpha_0,\theta_0),\quad
    \hat{\psi}_i(\theta)=\psi(W_i,\hat{\gamma}_{-\ell},\hat{\alpha}_{-\ell},\theta).
\end{align*}
\end{theorem}

\begin{proof}
We defer the proof to Supplement~\ref{sec:proof2}.
\end{proof}
    
    \bibliography{bib} 

\begin{thebibliography}{}

\bibitem[\protect\citeauthoryear{Abadie}{Abadie}{2003}]{abadie2003semiparametric}
Abadie, A. (2003).
\newblock Semiparametric instrumental variable estimation of treatment response
  models.
\newblock {\em Journal of Econometrics\/}~{\em 113\/}(2), 231--263.

\bibitem[\protect\citeauthoryear{Abdulkadiro{\u{g}}lu, Angrist, and
  Pathak}{Abdulkadiro{\u{g}}lu et~al.}{2014}]{abdulkadirouglu2014elite}
Abdulkadiro{\u{g}}lu, A., J.~Angrist, and P.~Pathak (2014).
\newblock The elite illusion: Achievement effects at {B}oston and {N}ew {Y}ork
  exam schools.
\newblock {\em Econometrica\/}~{\em 82\/}(1), 137--196.

\bibitem[\protect\citeauthoryear{Angrist and Evans}{Angrist and
  Evans}{1998}]{angrist1998children}
Angrist, J.~D. and W.~N. Evans (1998).
\newblock Children and their parents’ labor supply: Evidence from exogenous
  variation in family size.
\newblock {\em American Economic Review\/}~{\em 88\/}(3), 450--477.

\bibitem[\protect\citeauthoryear{Angrist and Fern{\'a}ndez-Val}{Angrist and
  Fern{\'a}ndez-Val}{2013}]{angrist2013extrapolate}
Angrist, J.~D. and I.~Fern{\'a}ndez-Val (2013).
\newblock Extrapo{LATE}-ing: External validity and overidentification in the
  {LATE} framework.
\newblock In {\em Advances in Economics and Econometrics}, pp.\  401--434.

\bibitem[\protect\citeauthoryear{Angrist, Imbens, and Rubin}{Angrist
  et~al.}{1996}]{angrist1996identification}
Angrist, J.~D., G.~W. Imbens, and D.~B. Rubin (1996).
\newblock Identification of causal effects using instrumental variables.
\newblock {\em Journal of the American Statistical Association\/}~{\em
  91\/}(434), 444--455.

\bibitem[\protect\citeauthoryear{Angrist and Pischke}{Angrist and
  Pischke}{2008}]{angrist2008mostly}
Angrist, J.~D. and J.-S. Pischke (2008).
\newblock {\em Mostly Harmless Econometrics: An Empiricist's Companion}.
\newblock Princeton University Press.

\bibitem[\protect\citeauthoryear{Baiocchi, Cheng, and Small}{Baiocchi
  et~al.}{2014}]{baiocchi2014instrumental}
Baiocchi, M., J.~Cheng, and D.~S. Small (2014).
\newblock Instrumental variable methods for causal inference.
\newblock {\em Statistics in Medicine\/}~{\em 33\/}(13), 2297--2340.

\bibitem[\protect\citeauthoryear{Belloni, Chen, Chernozhukov, and
  Hansen}{Belloni et~al.}{2012}]{belloni2012sparse}
Belloni, A., D.~Chen, V.~Chernozhukov, and C.~Hansen (2012).
\newblock Sparse models and methods for optimal instruments with an application
  to eminent domain.
\newblock {\em Econometrica\/}~{\em 80\/}(6), 2369--2429.

\bibitem[\protect\citeauthoryear{Belloni and Chernozhukov}{Belloni and
  Chernozhukov}{2013}]{belloni2013least}
Belloni, A. and V.~Chernozhukov (2013).
\newblock Least squares after model selection in high-dimensional sparse
  models.
\newblock {\em Bernoulli\/}~{\em 19\/}(2), 521--547.

\bibitem[\protect\citeauthoryear{Belloni, Chernozhukov, Fern{\'a}ndez-Val, and
  Hansen}{Belloni et~al.}{2017}]{belloni2017program}
Belloni, A., V.~Chernozhukov, I.~Fern{\'a}ndez-Val, and C.~Hansen (2017).
\newblock Program evaluation and causal inference with high-dimensional data.
\newblock {\em Econometrica\/}~{\em 85\/}(1), 233--298.

\bibitem[\protect\citeauthoryear{Bickel, Klaassen, Bickel, Ritov, Klaassen,
  Wellner, and Ritov}{Bickel et~al.}{1993}]{bickel1993efficient}
Bickel, P.~J., C.~A. Klaassen, P.~J. Bickel, Y.~Ritov, J.~Klaassen, J.~A.
  Wellner, and Y.~Ritov (1993).
\newblock {\em Efficient and Adaptive Estimation for Semiparametric Models},
  Volume~4.
\newblock Johns Hopkins University Press Baltimore.

\bibitem[\protect\citeauthoryear{Bickel, Ritov, and Tsybakov}{Bickel
  et~al.}{2009}]{bickel2009simultaneous}
Bickel, P.~J., Y.~Ritov, and A.~B. Tsybakov (2009).
\newblock Simultaneous analysis of {L}asso and {D}antzig selector.
\newblock {\em The Annals of Statistics\/}~{\em 37\/}(4), 1705--1732.

\bibitem[\protect\citeauthoryear{Blandhol, Bonney, Mogstad, and
  Torgovitsky}{Blandhol et~al.}{2022}]{blandhol2022tsls}
Blandhol, C., J.~Bonney, M.~Mogstad, and A.~Torgovitsky (2022).
\newblock When is {TSLS} actually {LATE}?
\newblock Technical report, National Bureau of Economic Research.

\bibitem[\protect\citeauthoryear{Chatterjee and Jafarov}{Chatterjee and
  Jafarov}{2015}]{chatterjee2015prediction}
Chatterjee, S. and J.~Jafarov (2015).
\newblock Prediction error of cross-validated lasso.
\newblock {\em arXiv:1502.06291\/}.

\bibitem[\protect\citeauthoryear{Cheng, Small, Tan, and Ten~Have}{Cheng
  et~al.}{2009}]{cheng2009efficient}
Cheng, J., D.~S. Small, Z.~Tan, and T.~R. Ten~Have (2009).
\newblock Efficient nonparametric estimation of causal effects in randomized
  trials with noncompliance.
\newblock {\em Biometrika\/}~{\em 96\/}(1), 19--36.

\bibitem[\protect\citeauthoryear{Chernozhukov, Chetverikov, Demirer, Duflo,
  Hansen, Newey, and Robins}{Chernozhukov
  et~al.}{2018}]{chernozhukov2018original}
Chernozhukov, V., D.~Chetverikov, M.~Demirer, E.~Duflo, C.~Hansen, W.~Newey,
  and J.~Robins (2018).
\newblock Double/debiased machine learning for treatment and structural
  parameters.
\newblock {\em The Econometrics Journal\/}~{\em 21\/}(1), C1--C68.

\bibitem[\protect\citeauthoryear{Chernozhukov, Escanciano, Ichimura, Newey, and
  Robins}{Chernozhukov et~al.}{2016}]{chernozhukov2016locally}
Chernozhukov, V., J.~C. Escanciano, H.~Ichimura, W.~K. Newey, and J.~M. Robins
  (2016).
\newblock Locally robust semiparametric estimation.
\newblock {\em arXiv:1608.00033, Econometrica (to appear)\/}.

\bibitem[\protect\citeauthoryear{Chernozhukov, Newey, and Singh}{Chernozhukov
  et~al.}{2018a}]{chernozhukov2018dantzig}
Chernozhukov, V., W.~Newey, and R.~Singh (2018a).
\newblock De-biased machine learning of global and local parameters using
  regularized {R}iesz representers.
\newblock {\em arXiv:1802.08667, Econometrics Journal (to appear)\/}.

\bibitem[\protect\citeauthoryear{Chernozhukov, Newey, and Singh}{Chernozhukov
  et~al.}{2018b}]{chernozhukov2018learning}
Chernozhukov, V., W.~K. Newey, and R.~Singh (2018b).
\newblock Automatic debiased machine learning of causal and structural effects.
\newblock {\em 1809.05224, Econometrica (to appear)\/}.

\bibitem[\protect\citeauthoryear{Chernozhukov, Newey, and Singh}{Chernozhukov
  et~al.}{2021}]{chernozhukov2021simple}
Chernozhukov, V., W.~K. Newey, and R.~Singh (2021).
\newblock A simple and general debiased machine learning theorem with finite
  sample guarantees.
\newblock {\em arXiv:2105.15197\/}.

\bibitem[\protect\citeauthoryear{Crump, Hotz, Imbens, and Mitnik}{Crump
  et~al.}{2009}]{crump_dealing_2009}
Crump, R.~K., V.~J. Hotz, G.~W. Imbens, and O.~A. Mitnik (2009, March).
\newblock Dealing with limited overlap in estimation of average treatment
  effects.
\newblock {\em Biometrika\/}~{\em 96\/}(1), 187--199.

\bibitem[\protect\citeauthoryear{Farrell, Liang, and Misra}{Farrell
  et~al.}{2021}]{farrell2021deep}
Farrell, M.~H., T.~Liang, and S.~Misra (2021).
\newblock Deep neural networks for estimation and inference.
\newblock {\em Econometrica\/}~{\em 89\/}(1), 181--213.

\bibitem[\protect\citeauthoryear{Fr{\"o}lich}{Fr{\"o}lich}{2007}]{frolich2007nonparametric}
Fr{\"o}lich, M. (2007).
\newblock Nonparametric {IV} estimation of local average treatment effects with
  covariates.
\newblock {\em Journal of Econometrics\/}~{\em 139\/}(1), 35--75.

\bibitem[\protect\citeauthoryear{Hasminskii and Ibragimov}{Hasminskii and
  Ibragimov}{1979}]{hasminskii1979nonparametric}
Hasminskii, R.~Z. and I.~A. Ibragimov (1979).
\newblock On the nonparametric estimation of functionals.
\newblock In {\em Proceedings of the Second Prague Symposium on Asymptotic
  Statistics}.

\bibitem[\protect\citeauthoryear{Hirshberg and Wager}{Hirshberg and
  Wager}{2021}]{hirshberg2021augmented}
Hirshberg, D.~A. and S.~Wager (2021).
\newblock Augmented minimax linear estimation.
\newblock {\em The Annals of Statistics\/}~{\em 49\/}(6), 3206--3227.

\bibitem[\protect\citeauthoryear{Hong and Nekipelov}{Hong and
  Nekipelov}{2010}]{hong2010semiparametric}
Hong, H. and D.~Nekipelov (2010).
\newblock Semiparametric efficiency in nonlinear {LATE} models.
\newblock {\em Quantitative Economics\/}~{\em 1\/}(2), 279--304.

\bibitem[\protect\citeauthoryear{Klaassen}{Klaassen}{1987}]{klaassen1987consistent}
Klaassen, C.~A. (1987).
\newblock Consistent estimation of the influence function of locally
  asymptotically linear estimators.
\newblock {\em The Annals of Statistics\/}, 1548--1562.

\bibitem[\protect\citeauthoryear{Luedtke and van~der Laan}{Luedtke and van~der
  Laan}{2016}]{luedtke2016statistical}
Luedtke, A.~R. and M.~J. van~der Laan (2016).
\newblock Statistical inference for the mean outcome under a possibly
  non-unique optimal treatment strategy.
\newblock {\em Annals of Statistics\/}~{\em 44\/}(2), 713.

\bibitem[\protect\citeauthoryear{Marbach and Hangartner}{Marbach and
  Hangartner}{2020}]{marbach2020profiling}
Marbach, M. and D.~Hangartner (2020).
\newblock Profiling compliers and noncompliers for instrumental-variable
  analysis.
\newblock {\em Political Analysis\/}~{\em 28\/}(3), 435--444.

\bibitem[\protect\citeauthoryear{Newey}{Newey}{1994}]{newey1994asymptotic}
Newey, W.~K. (1994).
\newblock The asymptotic variance of semiparametric estimators.
\newblock {\em Econometrica\/}, 1349--1382.

\bibitem[\protect\citeauthoryear{Newey and McFadden}{Newey and
  McFadden}{1994}]{newey1994large}
Newey, W.~K. and D.~McFadden (1994).
\newblock Large sample estimation and hypothesis testing.
\newblock {\em Handbook of Econometrics\/}~{\em 4}, 2111--2245.

\bibitem[\protect\citeauthoryear{Ogburn, Rotnitzky, and Robins}{Ogburn
  et~al.}{2015}]{ogburn_doubly_2015}
Ogburn, E.~L., A.~Rotnitzky, and J.~M. Robins (2015).
\newblock Doubly robust estimation of the local average treatment effect curve.
\newblock {\em Journal of the Royal Statistical Society: Series B (Statistical
  Methodology)\/}~{\em 77\/}(2), 373--396.

\bibitem[\protect\citeauthoryear{Robins and Rotnitzky}{Robins and
  Rotnitzky}{1995}]{robins1995semiparametric}
Robins, J.~M. and A.~Rotnitzky (1995).
\newblock Semiparametric efficiency in multivariate regression models with
  missing data.
\newblock {\em Journal of the American Statistical Association\/}~{\em
  90\/}(429), 122--129.

\bibitem[\protect\citeauthoryear{Robinson}{Robinson}{1988}]{robinson1988root}
Robinson, P.~M. (1988).
\newblock Root-n-consistent semiparametric regression.
\newblock {\em Econometrica\/}, 931--954.

\bibitem[\protect\citeauthoryear{Rotnitzky, Smucler, and Robins}{Rotnitzky
  et~al.}{2021}]{rotnitzky2021characterization}
Rotnitzky, A., E.~Smucler, and J.~M. Robins (2021).
\newblock Characterization of parameters with a mixed bias property.
\newblock {\em Biometrika\/}~{\em 108\/}(1), 231--238.

\bibitem[\protect\citeauthoryear{Schmidt-Hieber}{Schmidt-Hieber}{2020}]{schmidt2020nonparametric}
Schmidt-Hieber, J. (2020).
\newblock Nonparametric regression using deep neural networks with {ReLU}
  activation function.
\newblock {\em The Annals of Statistics\/}~{\em 48\/}(4), 1875--1897.

\bibitem[\protect\citeauthoryear{Singh and Sun}{Singh and
  Sun}{2019}]{singh2019biased}
Singh, R. and L.~Sun (2019).
\newblock De-biased machine learning in instrumental variable models for
  treatment effects.
\newblock {\em arXiv:1909.05244\/}.

\bibitem[\protect\citeauthoryear{S{\l}oczy{\'n}ski and
  Wooldridge}{S{\l}oczy{\'n}ski and Wooldridge}{2018}]{sloczynski2018general}
S{\l}oczy{\'n}ski, T. and J.~M. Wooldridge (2018).
\newblock A general double robustness result for estimating average treatment
  effects.
\newblock {\em Econometric Theory\/}~{\em 34\/}(1), 112--133.

\bibitem[\protect\citeauthoryear{Swanson and Hern{\'a}n}{Swanson and
  Hern{\'a}n}{2013}]{swanson2013commentary}
Swanson, S.~A. and M.~A. Hern{\'a}n (2013).
\newblock Commentary: How to report instrumental variable analyses (suggestions
  welcome).
\newblock {\em Epidemiology\/}~{\em 24\/}(3), 370--374.

\bibitem[\protect\citeauthoryear{Tan}{Tan}{2006}]{tan2006regression}
Tan, Z. (2006).
\newblock Regression and weighting methods for causal inference using
  instrumental variables.
\newblock {\em Journal of the American Statistical Association\/}~{\em
  101\/}(476), 1607--1618.

\bibitem[\protect\citeauthoryear{van~der Laan and Rose}{van~der Laan and
  Rose}{2018}]{van2018targeted}
van~der Laan, M.~J. and S.~Rose (2018).
\newblock {\em Targeted Learning in Data Science}.
\newblock Springer.

\bibitem[\protect\citeauthoryear{van~der Laan and Rubin}{van~der Laan and
  Rubin}{2006}]{van2006targeted}
van~der Laan, M.~J. and D.~Rubin (2006).
\newblock Targeted maximum likelihood learning.
\newblock {\em The International Journal of Biostatistics\/}~{\em 2\/}(1).

\bibitem[\protect\citeauthoryear{Zheng and van~der Laan}{Zheng and van~der
  Laan}{2011}]{zheng2011cross}
Zheng, W. and M.~J. van~der Laan (2011).
\newblock Cross-validated targeted minimum-loss-based estimation.
\newblock In {\em Targeted Learning}, pp.\  459--474. Springer Science \&
  Business Media.

\end{thebibliography}

\section*{Supplementary material}
Supplementary material includes proofs, rate conditions, simulations, implementation details, and code.

\section{Rate conditions}\label{sec:rate}

In this section, we present assumptions to guarantee that the estimators $(\hat{\gamma},\hat{\alpha})$ of the nonparametric functions $(\gamma_0,\alpha_0)$ satisfy the rate conditions in Assumption~\ref{kappa_regularity}. First, we place a weak assumption on the dictionary of basis functions $b$.
\begin{assumption}[Bounded dictionary]\label{B_b}
The dictionary is bounded. Formally, there exists some $C>0$ such that
$\max_j |b_j(Z,X)|\leq C$ almost surely.
\end{assumption}

Next, we articulate assumptions required for convergence of $\hat{\alpha}$ under two regimes: the regime in which $\alpha_0$ is dense and the regime in which $\alpha_0$ is sparse.
\begin{assumption}[Dense balancing weight]\label{dense_RR}
The balancing weight $\alpha_0$ is well approximated by the full dictionary $b$. Formally, assume there exist some $\rho_n\in\mathbb{R}^p$ and $C<\infty$ such that $|\rho_n|_1\leq C$ and $\|\alpha_0-b^{\top}\rho_n\|^2=O\{(\log p / n)^{1/2}\}$.
\end{assumption}
Assumption~\ref{dense_RR} is satisfied if, for example, $\alpha_0$ is a linear combination of $b$. 

\begin{assumption}[Sparse balancing weight]\label{sparse_RR}
The balancing weight $\alpha_0$ is well approximated by a sparse subset of the dictionary $b$. Formally, assume
\begin{enumerate}
    \item There exist $C>1$ and $\xi>0$ such that for all $ \bar{s}\leq C\left(\log p / n\right)^{-1/(1+2\xi)}$, there exists some $\bar{\rho}\in\mathbb{R}^{p}$ with $|\bar{\rho}|_1\leq C$ and $\bar{s}$ nonzero elements such that $\|\alpha_0-b^{\top}\bar{\rho}\|^2\leq C\bar{s}^{-\xi}$.
    \item $G=E\{b(Z,X)b(Z,X)^{\top}\}$ has largest eigenvalue uniformly bounded in $n$.
    \item Denote $\mathcal{J}_{\rho}=support(\rho)$. There exists $k>3$ such that for $\rho=\rho_L,\bar{\rho}$ 
$$
\textsc{re}(k)=\inf_{\delta\in \Delta(\mathcal{J}_{\rho})}\frac{\delta^{\top}G\delta}{\sum_{j\in \mathcal{J}_{\rho}}\delta_j^2} >0,\quad \Delta(\mathcal{J}_{\rho})=\left(\delta\in\mathbb{R}^p:\delta\neq0,\sum_{j\in \mathcal{J}^c_{\rho}}|\delta_j|\leq k \sum_{j\in \mathcal{J}_{\rho}}|\delta_j|\right).
$$ 
\label{eigen2}
\item $\log p=O(\log n)$.
\end{enumerate}
\end{assumption}
Assumption~\ref{sparse_RR} is satisfied if, for example, $\alpha_0$ is sparse or approximately sparse \citep{chernozhukov2018learning}. The uniform bound on the largest eigenvalue of $G$ rules out the possibility that $G$ is an equal correlation matrix. \textsc{re} is the population version of the restricted eigenvalue condition \citep{bickel2009simultaneous}. It generalizes the familiar notion of no multicollinearity to the high dimensional setting. The final condition $\log p=O(\log n)$ rules out the possibility that $p=\exp(n)$; dimension cannot grow too much faster than sample size.

We adapt convergence guarantees from \cite{chernozhukov2018learning} for the balancing weight estimator $\hat{\alpha}$ in Algorithm~\ref{alg_stage1}. We obtain a slow rate for dense $\alpha_0$ and a fast rate for sparse $\alpha_0$. In both cases, we require the data driven regularization parameter $\lambda_n$ to approach $0$ slightly slower than $(\log p/n)^{1/2}$.
\begin{assumption}[Regularization]\label{regularization}
$\lambda_n=a_n(\log p / n)^{1/2}$ for some $a_n\rightarrow \infty$.
\end{assumption}
For example, one could set $a_n=\log\{\log(n)\}$ \citep{chatterjee2015prediction}. In Supplement~\ref{sec:tuning}, we provide and justify an iterative tuning procedure to determine data driven regularization parameter $\lambda_n$. The guarantees are as follows.

\begin{lemma}[Dense balancing weight rate]\label{dense_rate}
Under Assumptions~\ref{assumption:id},~\ref{B_b},~\ref{dense_RR}, and~\ref{regularization}, $$\|\hat{\alpha}-\alpha_0\|^2=O_p\left\{a_n\left(\frac{\log p}{n}\right)^{1/2}\right\},\quad |\hat{\rho}|_1=O_p(1).$$
\end{lemma}
\begin{lemma}[Sparse balancing weight rate]\label{sparse_rate}
Under Assumptions~\ref{assumption:id},~\ref{B_b},~\ref{sparse_RR}, and~\ref{regularization}, $$\|\hat{\alpha}-\alpha_0\|^2=O_p\left\{a_n^2\left(\frac{\log p}{n}\right)^{2\xi/(1+2\xi)}\right\},\quad |\hat{\rho}|_1=O_p(1).$$
\end{lemma}
See Supplement~\ref{sec:proof2} for the proofs. Whereas Lemma~\ref{dense_rate} does not require an explicit sparsity condition, Lemma~\ref{sparse_rate} does. When $\xi>1/2$, the rate in Lemma~\ref{sparse_rate} is faster than the rate in Lemma~\ref{dense_rate} for $a_n$ growing slowly enough. Interpreting the rate in Lemma~\ref{sparse_rate}, $n^{-2\xi/(1+2\xi)}$ is the well known rate of convergence if the identity of the nonzero components of $\bar{\rho}$ were known. The fact that their identity is unknown introduces a cost of $(\log p)^{2\xi/(1+2\xi)}$. The cost $a_n^2$ can be made arbitrarily small.

We place a rate assumption on the machine learning estimator $\hat{\gamma}$. It is a weak condition that allows $\hat{\gamma}$ to converge at a rate slower than $n^{-1/2}$. Importantly, it allows the analyst a broad variety of choices of machine learning estimators such as neural network or lasso. \cite{schmidt2020nonparametric,farrell2021deep} provide a rate for the former, while Lemmas~\ref{dense_rate} and~\ref{sparse_rate} provide rates for the latter, using the functional $b\mapsto E\{b(Z,X)V^{\top}\}$ instead.
\begin{assumption}[Regression rate]\label{gamma_hat}
$\|\hat{\gamma}-\gamma_0\|=O_p(n^{-d_{\gamma}})$ where
\begin{enumerate}
    \item In the dense balancing weight regime, $1/4\leq d_{\gamma} \leq 1/2$;
    \item In the sparse balancing weight regime, $1/2-\xi/(1+2\xi) \leq d_{\gamma} \leq 1/2$.
\end{enumerate}
\end{assumption}
These regime specific lower bounds on $d_{\gamma}$ are sufficient conditions for the product rate condition.

\begin{corollary}[Verifying rate condition]\label{product}
Suppose the conditions of Lemma~\ref{dense_rate} or Lemma~\ref{sparse_rate} hold as well as Assumption~\ref{gamma_hat}. Then the rate conditions of Assumption~\ref{kappa_regularity} hold: $|\hat{\alpha}|_{\infty}=O_p(1)$, $\|\hat{\alpha}-\alpha_0\|=o_p(1)$, $\|\hat{\gamma}-\gamma_0\|=o_p(1)$, and $\|\hat{\alpha}-\alpha_0\| \|\hat{\gamma}-\gamma_0\|=o_p(n^{-1/2})$.
\end{corollary}
The product rate condition in Corollary~\ref{product} formalizes the trade off in estimation error permitted in estimating $(\gamma_0,\alpha_0)$. In particular, faster convergence of $\hat{\alpha}$ permits slower convergence of $\hat{\gamma}$. Prior information about the balancing weight $\alpha_0$ used to estimate $\hat{\alpha}$, encoded by sparsity or perhaps by additional moment restrictions, can be helpful in this way. We will appeal to this product condition while proving statistical guarantees for complier parameters. 


\section{Proof of consistency and asymptotic normality for Auto-$\kappa$}\label{sec:proof2}

\subsection{Lemmas from previous work}

In this section, we prove consistency and asymptotic normality. For simplicity, we focus on the affine complier parameters of Definition~\ref{def:affine}. Corollary~\ref{cor:LATE} shows that this class that includes several popular complier parameters, including the leading case of average complier characteristics. The inference arguments can be generalized to the entire class in Definition~\ref{def:class}, including moments that are nonlinear in $\theta$, by introducing heavier notation and additional sample splitting for the nonlinear cases; see \cite{chernozhukov2016locally} for details.

We present the results in two subsections. In this subsection, we quote lemmas from previous work. In the next subsection, we present original arguments to prove consistency and asymptotic normality for our instrumental variable setting.

Consider the notation
\begin{align*}
    \psi(w,\gamma,\alpha,\theta)&=m(w,\gamma,\theta)+\phi(w,\gamma,\alpha,\theta);\\
    m(w,\gamma,\theta)&=A(\theta)\tilde{m}(w,\gamma); \\
    \tilde{m}(w,\gamma)&=\gamma(1,x)-\gamma(0,x); \\
    \phi(w,\gamma,\alpha,\theta)&=\alpha(z,x)A(\theta)\{v-\gamma(z,x)\}.
\end{align*}

\begin{definition} Define the following matrix $G\in\mathbb{R}^{p\times p}$ and the vector $M \in\mathbb{R}^p$:
\begin{align*}
    G&=E\{b(Z,X)b(Z,X)^{\top}\}, \\
    M&=E\{m(W,b,\theta_0)\}.
\end{align*}
\end{definition}

\begin{proposition}[Lemma C1 of \cite{chernozhukov2018learning}]\label{G_hat}
Under Assumption~\ref{B_b}, we have  $|\hat{G}-G|_{\infty}=O_p\{\left(\log p/n\right)^{1/2}\}$.
\end{proposition}

\begin{proposition}[Lemma 4 of \cite{chernozhukov2018learning}]\label{M_hat}
Under Assumptions~\ref{assumption:id} and~\ref{B_b}, we have $|\hat{M}-M|_{\infty}=O_p\{\left(\log p/n\right)^{1/2}\}$.
\end{proposition}

\begin{proof}[Proof of Lemma~\ref{dense_rate}]
Applying Proposition~\ref{G_hat} and Proposition~\ref{M_hat}, the proof follows \citet[Theorem 1]{chernozhukov2018learning}.
\end{proof}

\begin{proof}[Proof of Lemma~\ref{sparse_rate}]
Applying  Proposition~\ref{G_hat} and Proposition~\ref{M_hat}, the proof follows \citet[Theorem 3]{chernozhukov2018learning}. The argument that $|\hat{\rho}|_1=O_p(1)$ is analogous to \citet[Lemmas 2 and 3]{chernozhukov2018learning}.
\end{proof}

\begin{lemma}[Theorem 6 of \cite{chernozhukov2018learning}]\label{regularity}
Under Assumptions~\ref{assumption:id} and~\ref{kappa_regularity}, the following results hold.
\begin{enumerate}
    \item $E\{\tilde{m}(W,\gamma_0)^2\}<\infty$,
    \item $E[\{\tilde{m}(W,\gamma)-\tilde{m}(W,\gamma_0)\}^2] \leq C \|\gamma-\gamma_0\|^2$,
    \item $\max_j|\tilde{m}(W,b_j)-\tilde{m}(W,0)|\leq C$. 
\end{enumerate}
\end{lemma}

\begin{lemma}\label{asymptotic_linearity}
Suppose the conditions of Theorem~\ref{asymptotic_normality} hold. Then we have $$n^{-1/2}\sum_{\ell=1}^L\sum_{i\in I_{\ell}} \psi(W_i,\hat{\gamma}_{-\ell},\hat{\alpha}_{-\ell},\hat{\theta})=n^{-1/2} \sum_{i=1}^n \psi_0(W_i)+o_p(1),\quad \psi_0(W_i)=\psi(W_i,\gamma_0,\alpha_0,\theta_0).$$
\end{lemma}

\begin{proof}
The proof follows from \citet[Theorem 5]{chernozhukov2018learning}, appealing to Corollary~\ref{product} and Lemma~\ref{regularity}.
\end{proof}

\begin{lemma}[Theorem 2.1 \cite{newey1994large}]\label{extremum}
Consider $\hat{\theta}$ defined as $\argmin_{\theta\in\Theta} \hat{Q}(\theta)$, where $\hat{Q}:\Theta\rightarrow\mathbb{R}$ estimates $Q_0:\Theta\rightarrow\mathbb{R}$. If
\begin{enumerate}
    \item $\Theta$ is compact,
   \item $Q_0$ is continuous in $\theta$ over $\Theta$,
   \item $Q_0$ is uniquely maximized at $\theta_0$,
   \item $\sup_{\theta\in\Theta}|\hat{Q}(\theta)-Q_0(\theta)|=o_p(1)$,
\end{enumerate}
then $\hat{\theta}=\theta_0+o_p(1)$.
\end{lemma}

\subsection{Consistency and asymptotic normality}\label{sec:stage2}

\begin{proposition}\label{local4}
Suppose the conditions of Theorem~\ref{asymptotic_normality} hold. Then for each fold $I_{\ell}$ the following holds:
\begin{enumerate}
    \item $E[\{m(W,\hat{\gamma}_{-\ell},\theta_0)-m(W,\gamma_0,\theta_0)\}^2\mid  I_{-\ell}]=o_p(1)$,
    \item $E[\{\phi(W,\hat{\gamma}_{-\ell},\alpha_0,\theta_0)-\phi(W,\gamma_0,\alpha_0,\theta_0)\}^2\mid  I_{-\ell}]=o_p(1)$,
    \item $E[\{\phi(W,\gamma_0,\hat{\alpha}_{-\ell},\theta_0)-\phi(W,\gamma_0,\alpha_0,\theta_0)\}^2\mid  I_{-\ell}]=o_p(1)$.
\end{enumerate}
The notation $E(\cdot\mid I_{-\ell})$ means conditional on $W_{-\ell}=(W_i)_{i\notin I_{\ell}}$, i.e. observations not in fold $I_{\ell}$.
\end{proposition}

\begin{proof}
First observe that
\begin{align*}
    \phi(W,\hat{\gamma}_{-\ell},\alpha_0,\theta_0)-\phi(W,\gamma_0,\alpha_0,\theta_0)&=\alpha_0(z,x)A(\theta_0)\{\gamma_0(z,x)-\hat{\gamma}_{-\ell}(z,x)\}, \\
    \phi(W,\gamma_0,\hat{\alpha}_{-\ell},\theta_0)-\phi(W,\gamma_0,\alpha_0,\theta_0)&=\{\hat{\alpha}_{-\ell}(z,x)-\alpha_0(z,x)\}A(\theta_0)\{v-\gamma_0(z,x)\}.
\end{align*}
To lighten the proof, we slightly abuse notation as follows:
\begin{align*}
    \|\gamma_0-\hat{\gamma}_{-\ell}\|^2&=E[\{\gamma_0(Z,X)-\hat{\gamma}_{-\ell}(Z,X)\}^2\mid I_{\ell}];\\
    \|\alpha_0-\hat{\alpha}_{-\ell}\|^2&=E[\{\alpha(Z,X)-\hat{\alpha}_{-\ell}(Z,X)\}^2\mid I_{\ell}].
\end{align*}

\begin{enumerate}
    \item By Lemma~\ref{regularity}, the convergence holds due to  
    $\|\gamma_0-\hat{\gamma}_{-\ell}\|=o_p(1)$.
   
    \item By Assumption~\ref{gamma_hat} and Assumption~\ref{kappa_regularity}, we have
    $$
    \|\alpha_0 A(\theta_0)(\gamma_0-\hat{\gamma}_{-\ell})\|\leq C A(\theta_0)\|\gamma_0-\hat{\gamma}_{-\ell}\|=o_p(1).
    $$
    \item By Lemma~\ref{dense_rate} or Lemma~\ref{sparse_rate}, Assumption~\ref{kappa_regularity}, and law of iterated expectations with respect to $I_{-\ell}$, we have
    $$
    \| (\hat{\alpha}_{-\ell}-\alpha_0)A(\theta_0)\{v-\gamma_0(z,x)\}\|\leq \| \hat{\alpha}_{-\ell}-\alpha_0\|A(\theta_0) C \vec{1}=o_p(1)
    $$
    where $\vec{1}$ is the vector of ones.
\end{enumerate}
\end{proof}

\begin{proposition}\label{local5}
Suppose the conditions of Theorem~\ref{asymptotic_normality} hold. Then
\begin{align*}
    &n^{-1/2}\sum_{\ell=1}^L\sum_{i\in I_{\ell}}\{
    \phi(W_i,\hat{\gamma}_{-\ell},\hat{\alpha}_{-\ell},\theta_0)
    -\phi(W_i,\hat{\gamma}_{-\ell},\alpha_0,\theta_0)
    \\
    &\quad-\phi(W_i,\gamma_0,\hat{\alpha}_{-\ell},\theta_0)
    +\phi(W_i,\gamma_0,\alpha_0,\theta_0)\}=o_p(1).
\end{align*}
\end{proposition}

\begin{proof}
To begin, write
\begin{align*}
 &\phi(w,\hat{\gamma}_{-\ell},\hat{\alpha}_{-\ell},\theta_0)
    -\phi(w,\hat{\gamma}_{-\ell},\alpha_0,\theta_0)
    -\phi(w,\gamma_0,\hat{\alpha}_{-\ell},\theta_0)
    +\phi(w,\gamma_0,\alpha_0,\theta_0)\\
    &=-\{\hat{\alpha}_{-\ell}(z,x)-\alpha_0(z,x)\}A(\theta_0)\{\hat{\gamma}_{-\ell}(z,x)-\gamma_0(z,x)\}.
\end{align*}
Because convergence in first mean implies convergence in probability, it suffices to analyze
    \begin{align*}
        &E\left[\left\vert n^{-1/2} \sum_{\ell=1}^L\sum_{i\in I_{\ell}} -\{\hat{\alpha}_{-\ell}(Z_i,X_i)-\alpha_0(Z_i,X_i)\}A(\theta_0)\{\hat{\gamma}_{-\ell}(Z_i,X_i)-\gamma_0(Z_i,X_i)\} \right\vert \right]\\
        &\leq\sum_{\ell=1}^L E\left[n^{1/2}\frac{1}{n} \sum_{i\in I_{\ell}} \left\vert -\{\hat{\alpha}_{-\ell}(Z_i,X_i)-\alpha_0(Z_i,X_i)\}A(\theta_0)\{\hat{\gamma}_{-\ell}(Z_i,X_i)-\gamma_0(Z_i,X_i)\} \right\vert\right]\\
        &= \sum_{\ell=1}^L E\left(E\left[n^{1/2}\frac{1}{n} \sum_{i\in I_{\ell}} \left\vert\{\hat{\alpha}_{-\ell}(Z_i,X_i)-\alpha_0(Z_i,X_i)\}A(\theta_0)\{\hat{\gamma}_{-\ell}(Z_i,X_i)-\gamma_0(Z_i,X_i)\}\right\vert\mid I_{-\ell}\right]\right) \\
        &=\sum_{\ell=1}^L E\left(E\left[\left\vert n^{1/2}\frac{n_{\ell}}{n}\{\hat{\alpha}_{-\ell}(Z_i,X_i)-\alpha_0(Z_i,X_i)\}A(\theta_0)\{\hat{\gamma}_{-\ell}(Z_i,X_i)-\gamma_0(Z_i,X_i)\}\right\vert \mid I_{-\ell}\right]\right). 
    \end{align*}
    Applying H\"older's inequality elementwise  and Corollary~\ref{product}, we have convergence for each summand as follows:
    \begin{align*}
        &E\left[|n^{1/2}\frac{n_{\ell}}{n}\{\hat{\alpha}_{-\ell}(Z_i,X_i)-\alpha_0(Z_i,X_i)\}A(\theta_0)\{\hat{\gamma}_{-\ell}(Z_i,X_i)-\gamma_0(Z_i,X_i)\}| \mid I_{-\ell}\right]  \\
        &\leq E \left[|n^{1/2}\{\hat{\alpha}_{-\ell}(Z_i,X_i)-\alpha_0(Z_i,X_i)\}A(\theta_0)\{\hat{\gamma}_{-\ell}(Z_i,X_i)-\gamma_0(Z_i,X_i)\}|\mid I_{-\ell}\right] \\
        &\leq n^{1/2}\|\hat{\alpha}_{-\ell}-\alpha_0\| A(\theta_0) \|\hat{\gamma}_{-\ell}-\gamma_0\| \\
        &=o_p(1).
    \end{align*}
    In the penultimate step, we slightly abuse notation, using
\begin{align*}
    \|\gamma_0-\hat{\gamma}_{-\ell}\|^2&=E[\{\gamma_0(Z,X)-\hat{\gamma}_{-\ell}(Z,X)\}^2\mid I_{\ell}];\\
    \|\alpha_0-\hat{\alpha}_{-\ell}\|^2&=E[\{\alpha(Z,X)-\hat{\alpha}_{-\ell}(Z,X)\}^2\mid I_{\ell}].
\end{align*}
\end{proof}

\begin{proposition}\label{local6}
Under Assumption~\ref{assumption:id}, for each fold $I_{\ell}$, the following holds:
\begin{enumerate}
    \item $n^{1/2}E\{\psi(W,\hat{\gamma}_{-\ell},\alpha_0,\theta_0)\}=o_p(1);$
    \item $n^{1/2}E\{\phi(W,\gamma_0,\hat{\alpha}_{-\ell},\theta_0)\}=o_p(1).$
\end{enumerate}
\end{proposition}

\begin{proof}
To begin, write
\begin{align*}
    E\{\psi(W,\hat{\gamma}_{-\ell},\alpha_0,\theta_0)\}&=E[A(\theta_0)\{\hat{\gamma}_{-\ell}(1,X)-\hat{\gamma}_{-\ell}(0,X)\}+\alpha_0(Z,X)A(\theta_0)\{V-\hat{\gamma}_{-\ell}(Z,X)\}]; \\
    E\{\phi(W,\gamma_0,\hat{\alpha}_{-\ell},\theta_0)\}&=E[\hat{\alpha}_{-\ell}(Z,X)A(\theta_0)\{V-\gamma_0(Z,X)\}].
\end{align*}
\begin{enumerate}
    \item By Proposition~\ref{prop:rr}, $E\left\{\psi(W,\hat{\gamma}_{-\ell},\alpha_0,\theta_0)\mid I_{-\ell}\right\} =0$.  Applying the law of iterated expectations, we have $E\{\psi(W,\hat{\gamma}_{-\ell},\alpha_0,\theta_0)\}=0$.
    \item By law of iterated expectations, $E\left\{\phi(W,\gamma_0,\hat{\alpha}_{-\ell},\theta_0)\mid I_{-\ell}\right\} =0$.  Applying the law of iterated expectations, we have $E\{\psi(W,\hat{\gamma}_{-\ell},\alpha_0,\theta_0)\}=0$.
\end{enumerate}
\end{proof}

\begin{proposition}\label{local7}
Suppose the conditions of Theorem~\ref{asymptotic_normality} hold. Then
\begin{enumerate}
    \item The Jacobian $J$ exists.
    \item There exists a neighborhood $\mathcal{N}$ of $\theta_0$ with respect to $|\cdot|_2$ such that
\begin{enumerate}
    \item $\|\hat{\gamma}_{-\ell}-\gamma_0 \|=o_p(1)$;
        \item $\|\hat{\alpha}_{-\ell}-\alpha_0\|=o_p(1)$;
    \item For $\|\gamma-\gamma_0\|$ and $\|\alpha-\alpha_0\|$ small enough, $\psi(W_i,\gamma,\alpha,\theta)$ is differentiable in $\theta$ with probability approaching one;
    \item There exists $\zeta>0$ and $d(W)$ such that $E\{d(W)\}<\infty$ and for  $\|\gamma-\gamma_0\|$ small enough, 
    $$\left|\frac{\partial \psi(w,\gamma,\alpha,\theta)}{\partial \theta}-\frac{\partial \psi( w,\gamma,\alpha,\theta_0) }{\partial\theta}\right|_2\leq d(w)|\theta-\theta_0|_2^{\zeta}.$$
\end{enumerate}
\item For any fold $I_{\ell}$ and any components $(j,k)$ ,
$$E\left\{\left|\frac{\partial \psi_j(W,\hat{\gamma}_{-\ell},\hat{\alpha}_{-\ell},\theta_0)}{\partial \theta_k}-\frac{\partial \psi_j( W,\gamma_0,\alpha_0,\theta_0) }{\partial\theta_k}\right|\right\}=o_p(1).$$
\end{enumerate}
\end{proposition}

\begin{proof}
To begin, write
$$
\frac{\partial \psi(w,\gamma,\alpha,\theta)}{\partial\theta}=\frac{\partial A(\theta)}{\partial \theta}\{\gamma(1,x)-\gamma(0,x)\} +\alpha(z,x)\frac{\partial A(\theta)}{\partial \theta}\{v-\gamma(z,x)\}
$$
where $\partial A(\theta)/\partial \theta$ is a tensor consisting of 1s and 0s.

To lighten the proof, we slightly abuse notation as follows:
\begin{align*}
    \|\gamma_0-\hat{\gamma}_{-\ell}\|^2&=E[\{\gamma_0(Z,X)-\hat{\gamma}_{-\ell}(Z,X)\}^2\mid I_{\ell}];\\
    \|\alpha_0-\hat{\alpha}_{-\ell}\|^2&=E[\{\alpha(Z,X)-\hat{\alpha}_{-\ell}(Z,X)\}^2\mid I_{\ell}].
\end{align*}

\begin{enumerate}
    \item It suffices to show the second moment of the derivative is finite. By triangle inequality and Assumption~\ref{kappa_regularity} we have
\begin{align*}
    &\left\| \frac{\partial A(\theta_0)}{\partial \theta}\{\gamma_0(1,x)-\gamma_0(0,x)\} +\alpha_0(z,x)\frac{\partial A(\theta)}{\partial \theta}\{v-\gamma_0(z,x)\} \right\|\\
    &\quad\leq  \frac{\partial A(\theta_0)}{\partial \theta}\left\{\|\gamma_0(1,x)-\gamma_0(0,x)\|+CC'\right\}.
\end{align*}
To bound the right hand side, by Lemma~\ref{regularity} we have
$$
\|\gamma_0(1,x)-\gamma_0(0,x)\|\leq \|\gamma_0(1,x)\|+\|\gamma_0(0,x)\|\leq C\|\gamma_0\|<\infty.
$$
    \item
    \begin{enumerate}
        \item The convergence holds due to Assumption~\ref{gamma_hat}.
        \item The convergence holds due to  Lemma~\ref{dense_rate} or Lemma~\ref{sparse_rate}.
        \item Differentiability holds since $\partial \psi(w,\gamma,\alpha,\theta)/\partial\theta$ does not depend on $\theta$.
        \item The left hand side is exactly $\vec{0}$ since $\partial \psi(w,\gamma,\alpha,\theta)/\partial\theta$ does not depend on $\theta$.
    \end{enumerate}
    \item It suffices to analyze the difference
    \begin{align*}
        \xi&=\hat{\gamma}_{-\ell}(1,x)-\hat{\gamma}_{-\ell}(0,x)+\hat{\alpha}_{-\ell}(z,x)\{v-\hat{\gamma}_{-\ell}(z,x)\} \\
        &\quad-\left[\gamma_0(1,x)-\gamma_0(0,x)+\alpha_0(z,x)\{v-\gamma_0(z,x)\}\right] \\
        &=\hat{\gamma}_{-\ell}(1,x)-\gamma_0(1,x) \\
        &\quad-\hat{\gamma}_{-\ell}(0,x)+\gamma_0(0,x) \\
        &\quad+\hat{\alpha}_{-\ell}(z,x)\{v-\hat{\gamma}_{-\ell}(z,x)\}-\alpha_0(z,x)\{v-\hat{\gamma}_{-\ell}(z,x)\} \\
        &\quad+\alpha_0(z,x)\{v-\hat{\gamma}_{-\ell}(z,x)\}-\alpha_0(z,x)\{v-\gamma_0(z,x)\} \\
        &=\hat{\gamma}_{-\ell}(1,x)-\gamma_0(1,x) \\
        &\quad-\hat{\gamma}_{-\ell}(0,x)+\gamma_0(0,x) \\
        &\quad+\{\hat{\alpha}_{-\ell}(z,x)-\alpha_0(z,x)\}\{v-\gamma_0(z,x)\}\\
        &\quad+\{\hat{\alpha}_{-\ell}(z,x)-\alpha_0(z,x)\}\{\gamma_0(z,x)-\hat{\gamma}_{-\ell}(z,x)\}\\
        &\quad+\alpha_0(z,x)\{\gamma_0(z,x)-\hat{\gamma}_{-\ell}(z,x)\}
    \end{align*}
    where we use the decomposition
\begin{align*}
&\hat{\alpha}_{-\ell}(z,x)\{v-\hat{\gamma}_{-\ell}(z,x)\}-\alpha_0(z,x)\{v-\hat{\gamma}_{-\ell}(z,x)\}\\
&=\{\hat{\alpha}_{-\ell}(z,x)-\alpha_0(z,x)\}\{v-\gamma_0(z,x)+\gamma_0(z,x)-\hat{\gamma}_{-\ell}(z,x)\}.
\end{align*}
Hence
\begin{align*}
    E\left(|\xi|\right)
    &\leq E\left\{|\hat{\gamma}_{-\ell}(1,X)-\gamma_0(1,X)|\right\} \\
    &\quad +E\left\{|\hat{\gamma}_{-\ell}(0,X)-\gamma_0(0,X)|\right\} \\
    &\quad +E\left[|\{\hat{\alpha}_{-\ell}(Z,X)-\alpha_0(Z,X)\}\{V-\gamma_0(Z,X)\}|\right] \\
    &\quad +E\left[|\{\hat{\alpha}_{-\ell}(Z,X)-\alpha_0(Z,X)\}\{\gamma_0(Z,X)-\hat{\gamma}_{-\ell}(Z,X)\}|\right]\\
    &\quad +E\left[|\alpha_0(Z,X)\{\gamma_0(Z,X)-\hat{\gamma}_{-\ell}(Z,X)\}|\right].
\end{align*}
Consider the first term. Under Assumption~\ref{gamma_hat}, applying law of iterated expectation, Jensen's inequality, and Lemma~\ref{regularity}, we have
\begin{align*}
    E\left\{|\hat{\gamma}_{-\ell}(1,X)-\gamma_0(1,X)|\right\}
    &=E\left[E\left\{|\hat{\gamma}_{-\ell}(1,X)-\gamma_0(1,X)| \mid  I_{-\ell}\right\}\right]
    \\
    &\leq E\{\|\hat{\gamma}_{-\ell}(1,x)-\gamma_0(1,x)\|\}\\
    &\leq CE(\|\hat{\gamma}_{-\ell}-\gamma_0\|) \\ 
    &=o_p(1).
\end{align*}
Likewise for the second term. Consider the third term. Under Assumption~\ref{kappa_regularity}, applying law of iterated expectation, Lemma~\ref{dense_rate} or Lemma~\ref{sparse_rate}, and H\"older's inequality we have
\begin{align*}
    &E[|\{\hat{\alpha}_{-\ell}(Z,X)-\alpha_0(Z,X)\}\{V-\gamma_0(Z,X)\}|]\\
    &=E\left(E[|\{\hat{\alpha}_{-\ell}(Z,X)-\alpha_0(Z,X)\}\{V-\gamma_0(Z,X)\}| \mid  I_{-\ell}]\right) \\
    &
    \leq E\left\{\|\hat{\alpha}_{-\ell}-\alpha_0\|\|v-\gamma_0(z,x)\|\right\}\\
    &\leq C E(\|\hat{\alpha}_{-\ell}-\alpha_0\|) \\
    &=o_p(1).
\end{align*}
Consider the fourth term. By law of iterated expectations, H\"older's inequality, and Corollary~\ref{product}  we have
\begin{align*}
    &E\left[|\{\hat{\alpha}_{-\ell}(Z,X)-\alpha_0(Z,X)\}\{\gamma_0(Z,X)-\hat{\gamma}_{-\ell}(Z,X)\}|\right] \\
    &=E\left(E\left[|\{\hat{\alpha}_{-\ell}(Z,X)-\alpha_0(Z,X)\}\{\gamma_0(Z,X)-\hat{\gamma}_{-\ell}(Z,X)\}|\mid  I_{-\ell}\right]\right) \\
    &\leq E\left(\|\hat{\alpha}_{-\ell}-\alpha_0\|\|\gamma_0-\hat{\gamma}_{-\ell}\|\right) \\
    &=o_p(1).
\end{align*}
Consider the fifth term. By law of iterated expectations, Assumptions~\ref{gamma_hat} and~\ref{kappa_regularity}, and Jensen's inequality, we have
\begin{align*}
    E\left[|\alpha_0(Z,X)\{\gamma_0(Z,X)-\hat{\gamma}_{-\ell}(Z,X)\}|\right]
    &=E\left(E\left[|\alpha_0(Z,X)\{\gamma_0(Z,X)-\hat{\gamma}_{-\ell}(Z,X)\}|\mid  I_{-\ell}\right]\right) \\
    &\leq C E\left[E\left\{|\gamma_0(Z,X)-\hat{\gamma}_{-\ell}(Z,X)|\mid  I_{-\ell}\right\}\right]\\
    &\leq CE(\|\gamma_0-\hat{\gamma}_{-\ell}\|)\\
    &=o_p(1).
\end{align*}
\end{enumerate}
\end{proof}

\begin{proposition}\label{consistency}
Suppose the conditions of Theorem~\ref{asymptotic_normality} hold. Then $\hat{\theta}=\theta_0+o_p(1)$.
\end{proposition}

\begin{proof}
We verify the four conditions of Lemma~\ref{extremum} with
\begin{align*}
    Q_0(\theta)&=E\{\psi_0(\theta)\}^{\top}E\{\psi_0(\theta)\}, \\
    \hat{Q}(\theta)&=\left\{\frac{1}{n}\sum_{\ell=1}^L \sum_{i\in I_{\ell }}\hat{\psi}_i(\theta)\right\}^{\top}\frac{1}{n}\sum_{\ell=1}^L \sum_{i\in I_{\ell }}\hat{\psi}_i(\theta), \\
    \psi_0(\theta)&=\psi(W,\gamma_0,\alpha_0,\theta), \\
    \hat{\psi}_i(\theta)&= \psi(W_i,\hat{\gamma}_{-\ell},\hat{\alpha}_{-\ell},\theta).
\end{align*}
\begin{enumerate}
    \item The first condition follows from Assumption~\ref{kappa_regularity},
    \item The second condition follows from Corollary~\ref{cor:LATE}.
    \item The third condition follows from Corollary~\ref{cor:LATE}.
    \item Define
\begin{align*}
    \eta_0(w)&=\gamma_0(1,x)-\gamma_0(0,x)+\alpha_0(z,x)\{v-\gamma_0(z,x)\}\\
    \hat{\eta}_{-\ell}(w)&=\hat{\gamma}_{-\ell}(1,x)-\hat{\gamma}_{-\ell}(0,x)+\hat{\alpha}_{-\ell}(z,x)\{v-\hat{\gamma}_{-\ell}(z,x)\}.
\end{align*}
It follows that for $i\in I_{\ell}$,
\begin{align*}
    \psi_0(\theta)&=A(\theta)\eta_0(W),\quad E\{\psi_0(\theta)\}=A(\theta)E\{\eta_0(W)\};\\
    \hat{\psi}_i(\theta)&=A(\theta)\hat{\eta}_{-\ell}(W_i),\quad \frac{1}{n}\sum_{\ell=1}^L \sum_{i\in I_{\ell}} \hat{\psi}_i(\theta)=A(\theta)\frac{1}{n}\sum_{\ell=1}^L \sum_{i\in I_{\ell}} \hat{\eta}_{-\ell}(W_i).
\end{align*}
It suffices to show $n^{-1}\sum_{\ell=1}^L \sum_{i\in I_{\ell}} \hat{\eta}_{-\ell}(W_i) = E\{\eta_0(W)\}+o_p(1)$ since by continuous mapping theorem this implies that for all $\theta$ in $\Theta$, $n^{-1}\sum_{\ell=1}^L \sum_{i\in I_{\ell}} \hat{\psi}_i(\theta)=E\{\psi_0(\theta)\}+o_p(1)$ and hence $\hat{Q}(\theta)=Q_0(\theta)+o_p(1)$ uniformly.

We therefore turn to proving the sufficient condition. Write
\begin{align*}
&\frac{1}{n}\sum_{\ell=1}^L \sum_{i\in I_{\ell}} \hat{\eta}_{-\ell}(W_i)-E\{\eta_0(W)\}\\
&=\quad 
\frac{1}{n}\sum_{\ell=1}^L \sum_{i\in I_{\ell}} \{\hat{\eta}_{-\ell}(W_i)-\eta_0(W_i)\}+
\frac{1}{n}\sum_{\ell=1}^L \sum_{i\in I_{\ell}}\eta_0(W_i)-E\{\eta_0(W)\}.    
\end{align*}

Consider the initial terms. Denote $\xi_i=\hat{\eta}_{-\ell}(W_i)-\eta_0(W_i)$ as in Proposition~\ref{local7} item 3. We prove convergence in mean by
\begin{align*}
    E\left( \left|\frac{1}{n}\sum_{\ell=1}^L\sum_{i\in I_{\ell}}\xi_i\right|\right)
    &\leq \sum_{\ell=1}^L E\left(\frac{1}{n}\sum_{i\in I_{\ell}} |\xi_i|\right) \\
    &=\sum_{\ell=1}^L E\left\{E\left(\frac{1}{n}\sum_{i\in I_{\ell}} |\xi_i|\mid I_{-\ell} \right)\right\}\\
    &=\sum_{\ell=1}^L E\left\{\frac{n_{\ell}}{n} E(|\xi_i|\mid I_{-\ell} )\right\} \\
    &\leq \sum_{\ell=1}^L E\left\{ E(|\xi_i|\mid I_{-\ell} )\right\}  \\
    &=o_p(1)
\end{align*}
where the first inequality is due to triangle inequality, the second equality is due to the law of iterated expectations, and the rest echoes the proof of Proposition~\ref{local7} item 3.

Consider the latter terms. By the weak law of large numbers, if $E\{\eta_0(W)^2\}<\infty$ then
$$
\frac{1}{n}\sum_{\ell=1}^L \sum_{i\in I_{\ell}}\eta_0(W_i)-E\{\eta_0(W)\}=\frac{1}{n}\sum_{i=1}^n \eta_0(W_i)-E\{\eta_0(W)\}=o_p(1).
$$
To finish the argument, we verify $E\{\eta_0(W)^2\}=\|\eta_0\|^2<\infty$. By triangle inequality, Assumption~\ref{kappa_regularity}, and Lemma~\ref{regularity},
$$
  \|\eta_0\|=\|\gamma_0(1,x)-\gamma_0(0,x)+\alpha_0(z,x)\{v-\gamma_0(z,x)\}\| \leq \|\gamma_0(1,x)-\gamma_0(0,x)\|+CC'.
    $$
    To bound the right hand side, appeal to Lemma~\ref{regularity}:
$$
\|\gamma_0(1,x)-\gamma_0(0,x)\|\leq \|\gamma_0(1,x)\|+\|\gamma_0(0,x)\|\leq C\|\gamma_0\|<\infty.
$$
\end{enumerate}
\end{proof}

\begin{proposition}\label{local8}
Suppose the conditions of Theorem~\ref{asymptotic_normality} hold. Then the following holds.
\begin{enumerate}
    \item $\hat{\theta}=\theta_0+o_p(1)$, 
    \item $J^{\top}J$ is nonsingular,
    \item $E\{\psi_0(W)^2\}<\infty$,
    \item $E[\{
    \phi(W,\hat{\gamma}_{-\ell},\hat{\alpha}_{-\ell},\theta_0)
    -\phi(W,\hat{\gamma}_{-\ell},\alpha_0,\theta_0)
    -\phi(W,\gamma_0,\hat{\alpha}_{-\ell},\theta_0)
    +\phi(W,\gamma_0,\alpha_0,\theta_0)\}^2]=o_p(1)$.
\end{enumerate}
\end{proposition}

\begin{proof}
As in the proof of Proposition~\ref{local5}, we can write
\begin{align*}
 &\phi(w,\hat{\gamma}_{-\ell},\hat{\alpha}_{-\ell},\theta_0)
    -\phi(w,\hat{\gamma}_{-\ell},\alpha_0,\theta_0)
    -\phi(w,\gamma_0,\hat{\alpha}_{-\ell},\theta_0)
    +\phi(w,\gamma_0,\alpha_0,\theta_0) \\&=-\{\hat{\alpha}_{-\ell}(z,x)-\alpha_0(z,x)\}A(\theta_0)\{\hat{\gamma}_{-\ell}(z,x)-\gamma_0(z,x)\}.
\end{align*}
To lighten the proof, we slightly abuse notation as follows:
\begin{align*}
    \|\gamma_0-\hat{\gamma}_{-\ell}\|^2&=E[\{\gamma_0(Z,X)-\hat{\gamma}_{-\ell}(Z,X)\}^2\mid I_{\ell}];\\
    \|\alpha_0-\hat{\alpha}_{-\ell}\|^2&=E[\{\alpha(Z,X)-\hat{\alpha}_{-\ell}(Z,X)\}^2\mid I_{\ell}].
\end{align*}
\begin{enumerate}
    \item Convergence holds due to Proposition~\ref{consistency}.
    \item Nonsingularity holds due to Assumption~\ref{kappa_regularity}.
    \item $E\{\psi_0(W)^2\}<\infty$ is immediate from $E\{\eta_0(W)^2\}$, which is proved in Proposition~\ref{consistency} item 4.
    \item It suffices to analyze
    \begin{align*}
        &E\left(\left[\{\hat{\alpha}_{-\ell}(z,x)-\alpha_0(z,x)\}A(\theta_0)\{\hat{\gamma}_{-\ell}(z,x)-\gamma_0(z,x)\}\right]^2\right) \\
        &=E\left\{E\left(\left[\{\hat{\alpha}_{-\ell}(z,x)-\alpha_0(z,x)\}A(\theta_0)\{\hat{\gamma}_{-\ell}(z,x)-\gamma_0(z,x)\}\right]^2\mid  I_{-\ell}\right)\right\} \\
        &=
        E\left\{\|(\hat{\alpha}_{-\ell}-\alpha_0)A(\theta_0)(\hat{\gamma}_{-\ell}-\gamma_0)\|^2\right\}\\
        &\leq 2E\left\{\|\hat{\alpha}_{-\ell}A(\theta_0)(\hat{\gamma}_{-\ell}-\gamma_0)\|^2+\|\alpha_0A(\theta_0)(\hat{\gamma}_{-\ell}-\gamma_0)\|^2\right\}.
    \end{align*}
        Consider the first term. By H\"older's inequality, Assumption~\ref{B_b}, and either Lemma~\ref{dense_rate} or Lemma~\ref{sparse_rate}, we have
        $$
        |\hat{\alpha}_{-\ell}(z,x)|=|\hat{\rho}_{-\ell}^{\top}b(z,x)|\leq |\hat{\rho}_{-\ell}|_1|b(z,x)|_{\infty} =O_p(1).
        $$
        It follows by Assumption~\ref{gamma_hat} that
        $$
        \|\hat{\alpha}_{-\ell}A(\theta_0)(\hat{\gamma}_{-\ell}-\gamma_0)\|=O_p(1) \|\hat{\gamma}_{-\ell}-\gamma_0\|=O_p(n^{-d_{\gamma}})=o_p(1).
        $$
       Consider the second term. By Assumption~\ref{gamma_hat} and Assumption~\ref{kappa_regularity}, we have
    $$
    \|\alpha_0A(\theta_0)(\hat{\gamma}_{-\ell}-\gamma_0)\|\leq CA(\theta_0) \|\hat{\gamma}_{-\ell}-\gamma_0\|=o_p(1).
    $$
\end{enumerate}
\end{proof}

\begin{proof}[Proof of Theorem~\ref{asymptotic_normality}]
The proof now follows from \citet[Theorems 16 and 17]{chernozhukov2016locally}. In particular, Proposition~\ref{local4} verifies \citet[Assumption 4]{chernozhukov2016locally}, Proposition~\ref{local5} verifies \citet[Assumption 5]{chernozhukov2016locally}, Proposition~\ref{local6} verifies \citet[Assumption 6]{chernozhukov2016locally}, Proposition~\ref{local7} verifies \citet[Assumption 7]{chernozhukov2016locally}, and Proposition~\ref{local8} verifies the additional conditions in \citet[Theorems 16 and 17]{chernozhukov2016locally}. Finally, the parameter $\theta_0$ is exactly identified; $J$ is a square matrix, the GMM weighting can be taken as the identity matrix, so the formula for the asymptotic covariance matrix simplifies.
\end{proof}

\section{Tuning}\label{sec:tuning}

Algorithm~\ref{alg_stage1} takes as given the value of regularization parameter $\lambda_n$. For practical use, we provide an iterative tuning procedure to empirically determine $\lambda_n$. This is precisely the tuning procedure of \cite{chernozhukov2018learning}, adapted from \cite{chernozhukov2018dantzig}. Due to its iterative nature, the tuning procedure is most clearly stated as a replacement for Algorithm~\ref{alg_stage1}.

Recall that the inputs to Algorithm~\ref{alg_stage1} are observations in $I_{-\ell}$, i.e. excluding fold $\ell$. The analyst must also specify the $p$ dimensional dictionary $b$. For notational convenience, we assume $b$ includes the intercept in its first component: $b_1(z,x)=1$. In this tuning procedure, the analyst must further specify a low dimensional subdictionary $b^{\text{low}}$ of $b$. As in Algorithm~\ref{alg_stage1}, the output of the tuning procedure is $\hat{\alpha}_{-\ell}$, an estimator of the balancing weight trained only on observations in $I_{-\ell}$.

The tuning procedure is as follows.
\begin{algorithm}[Regularized balancing weight with tuning]\label{alg_tuning}
For observations in $I_{-\ell}$,
\begin{enumerate}
    \item Initialize $\hat{\rho}_{-\ell}$ using $b^{\text{low}}$: 
\begin{align*}
    \hat{G}^{\text{low}}_{-\ell}&=\frac{1}{n-n_{\ell}}\sum_{i\in I_{-\ell}} b^{\text{low}}(Z_i,X_i)b^{\text{low}}(Z_i,X_i)^{\top}; \\
    \hat{M}^{\text{low}}_{-\ell}&=\frac{1}{n-n_{\ell}}\sum_{i\in I_{-\ell}} b^{\text{low}}(1,X_i)-b^{\text{low}}(0,X_i); \\
    \hat{\rho}_{-\ell}&=\begin{Bmatrix} \left(\hat{G}^{\text{low}}_{-\ell}\right)^{-1}\hat{M}^{\text{low}}_{-\ell} \\ 0\end{Bmatrix}.
\end{align*}
    \item Calculate moments 
    \begin{align*}
    \hat{G}_{-\ell}&=\frac{1}{n-n_{\ell}}\sum_{i\in I_{-\ell}} b(Z_i,X_i)b(Z_i,X_i)^{\top}; \\
    \hat{M}_{-\ell}&=\frac{1}{n-n_{\ell}}\sum_{i\in I_{-\ell}} b(1,X_i)-b(0,X_i). 
\end{align*}
    \item While $\hat{\rho}_{-\ell}$ has not converged,
    \begin{enumerate}
  \item Update normalization
  $$
  \hat{D}_{-\ell}=\left(\frac{1}{n-n_{\ell}}\sum_{i\in I_{-\ell}} 
  [b(Z_i,X_i) b(Z_i,X_i)^{\top}\hat{\rho}_{-\ell}- \{b(1,X_i)-b(0,X_i)\}]^2
  \right)^{1/2}.
  $$
    \item Update $(\lambda_n,\hat{\rho}_{-\ell})$
    \begin{align*}
        \lambda_n&=\frac{c_1}{(n-n_{\ell})^{1/2}}\Phi^{-1}\left(1-\frac{c_2}{2p}\right); \\
        \hat{\rho}_{-\ell}&=\argmin_{\rho} \rho^{\top}\hat{G}_{-\ell}\rho-2\rho^{\top} \hat{M}_{-\ell}+2\lambda_n c_3 |\hat{D}_{-\ell,11} \rho_1|+2\lambda_n \sum_{j=2}^p |\hat{D}_{-\ell,jj}\rho_j|;
    \end{align*}
    where $\rho_j$ is the $j$th coordinate of $\rho$ and $\hat{D}_{-\ell,jj}$ is the $j$th diagonal entry of $\hat{D}_{-\ell}$.
    \end{enumerate}
    \item Set $\hat{\alpha}_{-\ell}(z,x)=b(z,x)^{\top}\hat{\rho}_{-\ell}$.
\end{enumerate}
\end{algorithm}

In step 1, $b^{low}$ is sufficiently low dimensional that $\hat{G}_{-\ell}^{\text{low}}$ is invertible. In practice, we take $dim(b^{low})=dim(b)/40$. 

In step 3,  $(c_1,c_2,c_3)$ are hyperparameters taken as $(1/2,0.1,0.1)$ in practice. We implement the optimization via generalized coordinate descent with soft thresholding. See \cite{chernozhukov2018learning} for a detailed derivation of this soft thresholding routine. In the optimization, we initialize at the previous value of $\hat{\rho}_{-\ell}$. For numerical stability, we use $\hat{D}_{-\ell}+0.2 I$ instead of $\hat{D}_{-\ell}$, and we cap the maximum number of iterations at 10.

We justify Algorithm~\ref{alg_tuning} in the same manner as \citet[Section 5.1]{chernozhukov2018dantzig}. Specifically, we appeal to \citet[Theorem 8]{belloni2013least} for the homoscedastic case and \citet[Theorem 1]{belloni2012sparse} for the heteroscedastic case.
\section{Simulations}\label{sec:sim}

\subsection{Simultaneous confidence band}

Suppose we wish to form a simultaneous confidence band for the components of $\hat{\theta}$, which may be the complier counterfactual outcome distribution based on a finite grid $\mathcal{U}$, which is a subset of $\mathcal{Y}$. The following procedure allows us to do so from some estimator $\hat{C}$ for the asymptotic variance $C$ of $\hat{\theta}$.  Let  $\hat{S}=diag(\hat{C})$ and $S=diag(C)$ collect the diagonal elements of these matrices.

\begin{algorithm}[Simultaneous confidence band]\label{alg_simultaneous}
Given $\hat{C}$,
\begin{enumerate}
    \item Calculate $\hat{\Sigma}=\hat{S}^{-1/2}\hat{C}\hat{S}^{-1/2}$.
    \item Sample $Q$ from $\mathcal{N}(0,\hat{\Sigma})$ and compute the value $\hat{c}_{a}$ as the $(1-a)$ quantile of sampled $|Q|_{\infty}$.
    \item Form the confidence band $$
    (l_j,u_j)=\left\{\hat{\theta}_j-\hat{c}_{a}(\hat{C}_{jj}/n)^{1/2},\ \hat{\theta}_j+\hat{c}_{a}(\hat{C}_{jj}/n)^{1/2}\right\}$$ 
    where $\hat{C}_{jj}$ is the diagonal entry of $\hat{C}$ corresponding to $j$th element $\hat{\theta}_j$ of $\theta$.
\end{enumerate}
\end{algorithm}

\begin{corollary}[Simultaneous confidence band]\label{thm_simultaneous} Suppose the
conditions of Theorem~\ref{asymptotic_normality} hold. Then for a fixed and finite grid $\mathcal{U}$, the confidence band in Algorithm~\ref{alg_simultaneous}
jointly covers the true counterfactual distributions $\theta_0$ at all grid points $y$ in $\mathcal{U}$ with
probability approaching the nominal level, i.e.
$
\text{\normalfont pr}\{(\theta_0)_j\in(l_j,u_j) \; \text{ for all } j\}\rightarrow 1-a.
$
\end{corollary}

\begin{proof}
Let $c_a$ be the $(1-a)$ quantile of $|\mathcal{N}(0,\Sigma)|_{\infty}$ where $\Sigma = S^{-1/2} C S^{-1/2}$ and $S=diag(C)$.  We first show that this critical value ensures correct asymptotic simultaneous coverage of confidence bands in the form of the rectangle     
$$\{(l_0)_j,(u_0)_j\}=\left\{\hat{\theta}_j-c_a\left(\frac{{C}_{jj}}{n}\right)^{1/2},\ \hat{\theta}_j+c_a\left(\frac{{C}_{jj}}{n}\right)^{1/2}\right\}$$
where ${C}_{jj}$ is the diagonal entry of ${C}$ corresponding to $j$th element $\hat{\theta}_j$ of $\theta$.

The argument is as follows. Denote $(l_0,u_0)=\times^{2d}_{j=1} \{(l_0)_j,(u_0)_j\}$ where $d=dim(\mathcal{U})$. Then the simultaneous coverage probability is
\begin{align*}
\text{\normalfont pr}\{\theta_0 \text{ is in } (l_0,u_0)\}
&=\text{\normalfont pr}\{n^{1/2}(\hat{\theta}-\theta_0)\text{ is in } S^{1/2}(-c_a,c_a)^{2d}\} \\
&\rightarrow \text{\normalfont pr}\{\mathcal{N}(0,C)\text{ is in } S^{1/2}(-c_a,c_a)^{2d}\} \\
&=\text{\normalfont pr}\{S^{-1/2}\mathcal{N}(0,C)\text{ is in } (-c_a,c_a)^{2d}\} \\
&=\text{\normalfont pr}\{|\mathcal{N}(0,\Sigma)|_{\infty}\leq c_a\}  \\
&= 1-a.
\end{align*}
Gaussian multiplier bootstrap is operationally equivalent to approximating $c_a$ with $\hat{c}_a$, calculated in Algorithm~\ref{alg_simultaneous}, which is based on the consistent estimator $\hat{C}$. 
\end{proof}

\subsection{Results}

We compare the performance of our proposed Auto-$\kappa$ estimator with $\kappa$ weighting  \citep{abadie2003semiparametric} and the original debiased machine learning with explicit propensity scores \citep{chernozhukov2018original} in simulations. We focus on counterfactual distributions as our choice of complier parameter $\theta_0$ over the grid $\mathcal{U}$ specified on the horizontal axis of Figure~\ref{fig:simulation}.

\begin{figure}[h]
\begin{centering}
 \includegraphics[width=\textwidth]{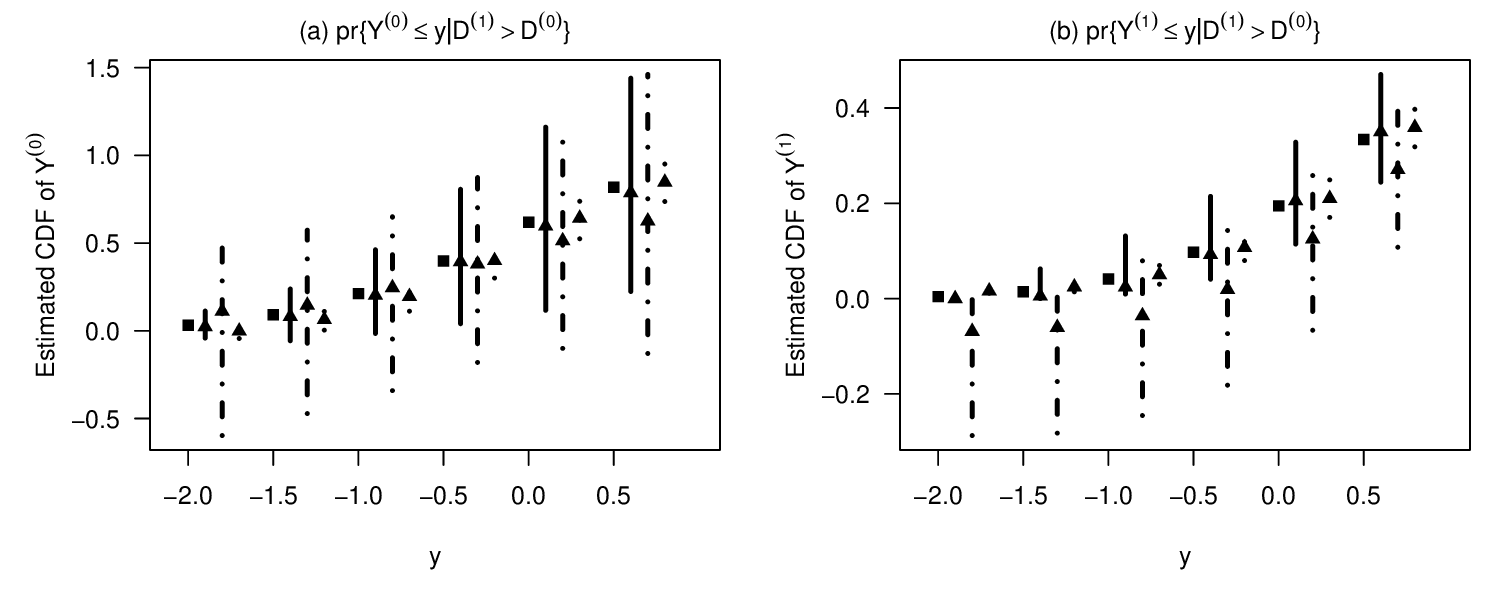}
 \end{centering}
\caption{\label{fig:simulation} Numerical stability simulation. Simulation performance of $\kappa$ weight (line), debiased machine learning (dot dash), and Auto-$\kappa$ (dots) estimators for the counterfactual distribution, where the grid point is specified on the horizontal axis.  The true values are solid squares. The vertical lines mark the 10\% and 90\% quantiles of the estimates across simulation draws and the solid triangles mark the median. }

\end{figure}

We consider a simple simulation design where  $Y$ is a continuous outcome, $D$ is a binary treatment, $Z$ is a binary instrumental variable,
and $X$ is a continuous covariate. We provide more details on the simulation design below. Each simulation consists of $n=1000$ observations. We conduct 1000 such simulations and implement each estimator as follows.

For the $\kappa$ weight, we estimate the propensity score $\hat{\pi}$ by logistic regression, which we then use in the weights $\hat{\kappa}^{(0)}(W),\hat{\kappa}^{(1)}(W)$ and subsequently the estimator $\hat{\theta}$. For debiased machine learning, we use five folds. We estimate the propensity score $\hat{\pi}$ by $\ell_1$ regularized logistic regression, using a dictionary of basis functions $b(X)$ consisting of fourth order polynomials of $X$. We estimate $\hat{\gamma}$ by lasso, using a dictionary of basis functions $b(Z,X)$ consisting of fourth order polynomials of $X$ and interactions between $Z$ and the polynomials. 

For Auto-$\kappa$, the key difference is that instead of estimating the propensity score, we directly estimate the balancing weight $\hat{\alpha}$ as described in Appendix~\ref{sec:estimation}, using a dictionary of basis functions $b(Z,X)$ consisting of fourth order polynomials of $X$ and interactions between $Z$ and the polynomials. Subsequently, we estimate $\hat{\theta}$ and construct simultaneous confidence bands by steps outlined above. Since the true propensity scores $\pi_0(X)$ are highly nonlinear, we expect $\kappa$ weighting and debiased machine learning to encounter issues of numerical instability. Furthermore, $\kappa$ weighting might not be as efficient as the debiased machine learning and Auto-$\kappa$ estimators, which have the semiparametrically efficient asymptotic variance.

\begin{table}[h]
\caption{Bias and RMSE simulation for $\text{\normalfont pr}\{Y^{(0)}\leq y\mid D^{(1)}>D^{(0)}\}$}
\begin{centering}
\begin{tabular}{ccccccc}
 \\
 & \multicolumn{3}{c}{Bias} & \multicolumn{3}{c}{RMSE} \\  
$y$ & $\kappa$ weight & DML & Auto-$\kappa$ & $\kappa$ weight & DML & Auto-$\kappa$ \\[5pt]
-2.0 & -3  & -138  &  -37 &  99 & 3070  & 75  \\
-1.5 &  -1 &  -119 &  -32 &  172 & 2576  & 76   \\
-1.0 &  3 &  -45 &  -20 &  250 & 2040  & 79  \\
-0.5 &  2 &  -35 & 2  & 384  & 1953  & 80  \\
0.0 &  -17 & 18  & 21  &  556 & 1738  & 92  \\
0.5 &  -12 &  3 &  34 &  638 & 3072  &  98 \\
overall & -5 & -53 & -5  &  350 & 2391  &  83 
\end{tabular}\\
\end{centering}
\label{tab:simulation0}
\textit{Notes:}
RMSE, root mean square error; DML, debiased machine learning; Auto-$\kappa$, automatic $\kappa$ weighting. All entries have been multiplied by $10^3$.
\end{table}

\begin{table}[h]
\caption{Bias and RMSE simulation for $\text{\normalfont pr}\{Y^{(1)}\leq y\mid D^{(1)}>D^{(0)}\}$}
\begin{centering}
\begin{tabular}{ccccccc}
 \\
 & \multicolumn{3}{c}{Bias} & \multicolumn{3}{c}{RMSE} \\  
$y$ & $\kappa$ weight & DML & Auto-$\kappa$ & $\kappa$ weight & DML & Auto-$\kappa$ \\[5pt]
-2.0 &  2  &  -115  & 13  &  28  &  444  &  15  \\
-1.5 &   4 & -114   &  12  &  39  &  441  & 16    \\
-1.0 &  8  & -110   &  11  &  57  & 432   &  20  \\
-0.5 &  16  & -103   & 11   &  78  & 410   &  26  \\
0.0 & 21   & -93   &  16  &  90  & 379   &  35  \\
0.5 &  21  &  -79  &  27  &  92  &  315  &  44  \\
overall &  12 & -102  &  15  &  64  &  403  &  26 \\
\end{tabular}\\
\end{centering}
\label{tab:simulation1}
\textit{Notes:}
RMSE, root mean square error; DML, debiased machine learning; Auto-$\kappa$, automatic $\kappa$ weighting. All entries have been multiplied by $10^3$.

\end{table}

For each value in the grid $\mathcal{U}$, Tables~\ref{tab:simulation0} and~\ref{tab:simulation1} present the bias and the root mean square error (RMSE) of each estimator across simulation draws. The last row averages the performance across grid points. Figure~\ref{fig:simulation}  visualizes the median as well as the 10\% and 90\% quantiles across simulation draws. Auto-$\kappa$ outperforms debiased machine learning by a large margin due to numerical stability. Even though Auto-$\kappa$ uses regularized machine learning to estimate $(\hat{\gamma},\hat{\alpha})$, regularization bias does not translate into bias for estimating the counterfactual distribution due to the doubly robust moment function. In terms of efficiency, Auto-$\kappa$ substantially outperforms $\kappa$ weighting. Lastly, the simultaneous confidence bands based on the Auto-$\kappa$ estimator have coverage probability of 98.4\% for the counterfactual distribution of $Y^{(0)}$ and 93.6\% for the counterfactual distribution of $Y^{(1)}$, which are quite close to the nominal level of 95\%.
 
Numerical instability from inverting $\hat{\pi}$ is a known issue.  In practice, researchers may try trimming and censoring. Trimming means excluding observations for which $\hat{\pi}$ is extreme. We trim according to \cite{belloni2017program}, dropping observations with $\hat{\pi}$ not in $(10^{-12}, 1-10^{-12})$. Censoring means imposing bounds on $\hat{\pi}$ for such observations. We censor by setting $\hat{\pi}<10^{-12}$ to be $10^{-12}$ and $\hat{\pi}>1-10^{-12}$ to be $1-10^{-12}$. Auto-$\kappa$ without trimming or censoring outperforms $\kappa$ weighting and debiased machine learning even with trimming and censoring in this simulation design. Compare Figure~\ref{fig:simulation}, which has no preprocessing, with Figure~\ref{fig:simulation_trim}, which has trimming, and Figure~\ref{fig:simulation_censor}, which has censoring, to see this phenomenon. This property is convenient, since ad hoc trimming and censoring have limited theoretical justification \citep{crump_dealing_2009}.

\begin{figure}[h] 

\begin{centering}
 \includegraphics[width=\textwidth]{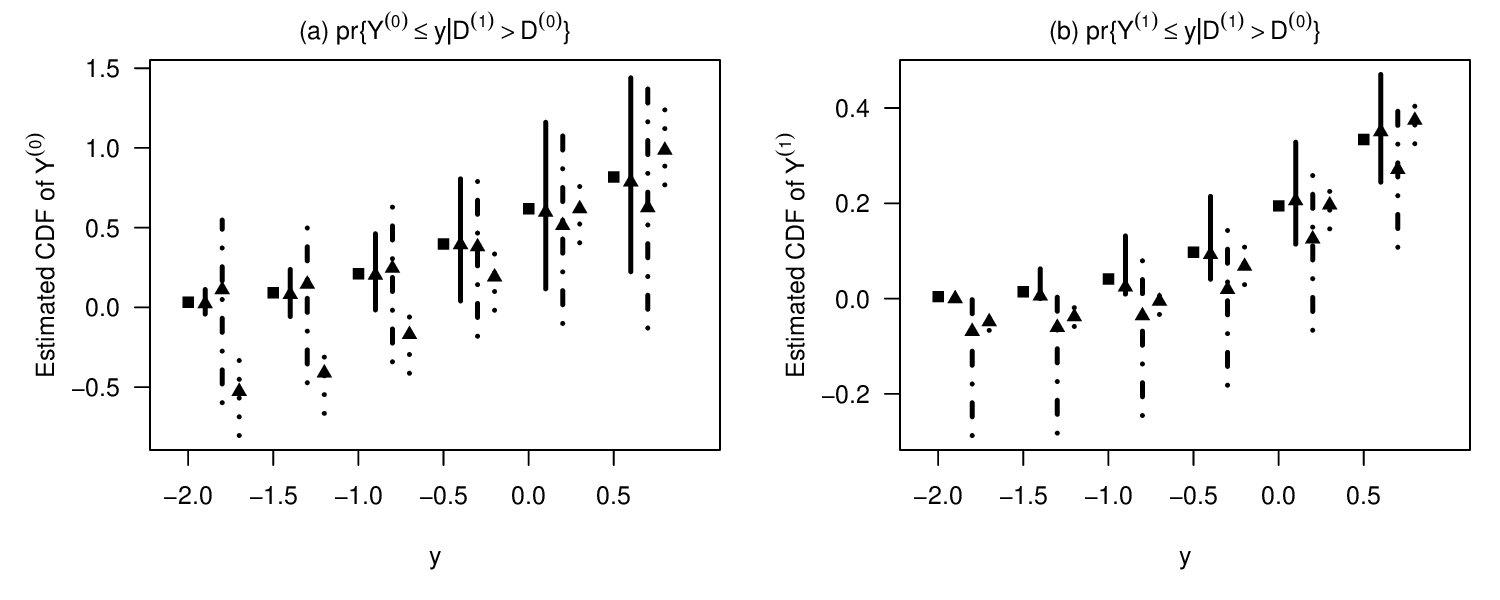}
 \end{centering}
\caption{\label{fig:simulation_trim}
Numerical stability simulation: Trimming. Simulation performance of $\kappa$ weight (line), debiased machine learning (dot dash), and Auto-$\kappa$ (dots) estimators for the counterfactual distribution, where the grid point is specified on the horizontal axis.  The true values are solid squares. The vertical lines mark the 10\% and 90\% quantiles of the estimates across simulation draws and the solid triangles mark the median. Observations with extreme propensity scores $\hat{\pi}$ not in $(10^{-12}, 1-10^{-12})$ are dropped.}
\end{figure}

\begin{figure}[h]

\begin{centering}
 \includegraphics[width=\textwidth]{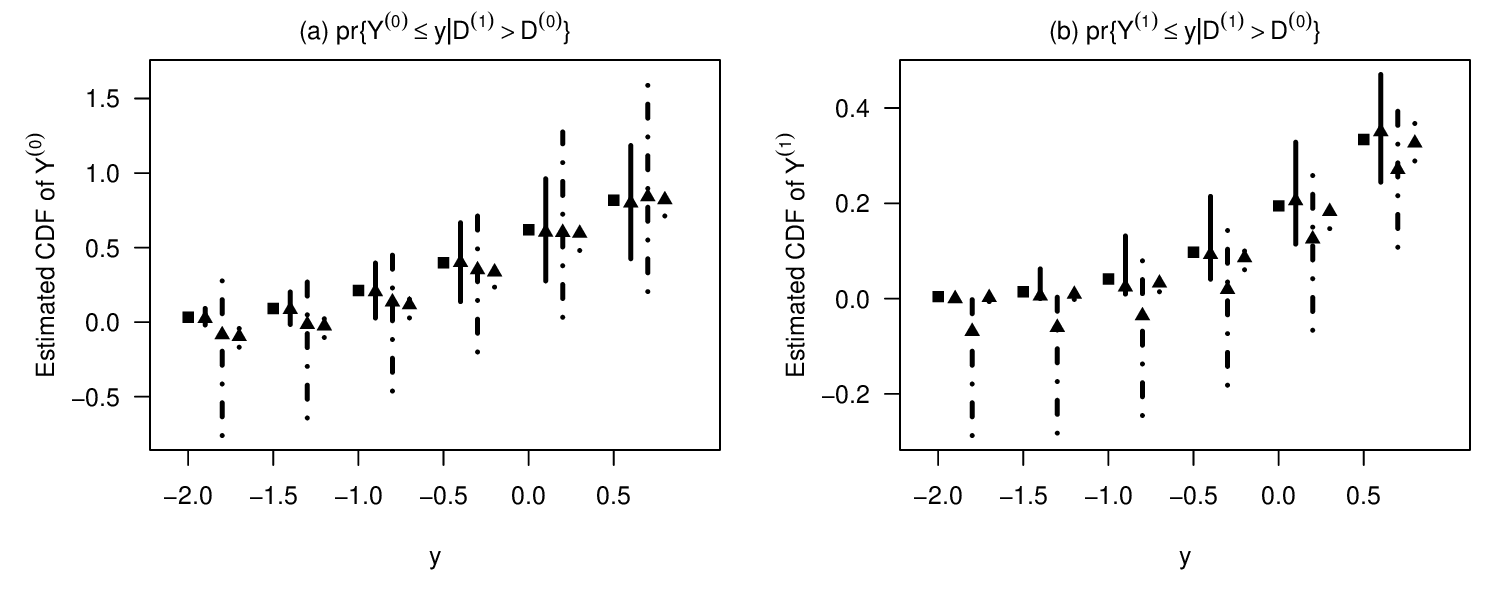}
 \end{centering}
\caption{\label{fig:simulation_censor}
Numerical stability simulation: Censoring. Simulation performance of $\kappa$ weight (line), debiased machine learning (dot dash), and Auto-$\kappa$ (dots) estimators for the counterfactual distribution, where the grid point is specified on the horizontal axis.  The true values are solid squares. The vertical lines mark the 10\% and 90\% quantiles of the estimates across simulation draws and the solid triangles mark the median. Observations with extreme propensity scores are censored by setting $\hat{\pi}<10^{-12}$ to be $10^{-12}$ and $\hat{\pi}>1-10^{-12}$ to be $1-10^{-12}$.}
\end{figure}

\subsection{Design}
Each simulation consists of a sample of $n=1000$  observations. A given observation is generated from the following model.
\begin{enumerate}
    \item Draw $X$ from $\mathcal{U}[0,1]$.
    \item Draw $Z\mid X=x$ from Bernoulli$\{\pi_0(x)\}$, where 
$
\pi_0(x) = (0.05) 1_{x\leq 0.5} + (0.95)  1_{x> 0.5}.
$
    \item Draw $D\mid Z=z,X=x$ from Bernoulli$(zx)$.
    \item Draw $Y\mid Z=z,X=x$ from $\mathcal{N}(2zx^{2},1)$.
\end{enumerate}
From observations of $W=(Y,D,Z,X^{\top})^{\top}$, we estimate complier counterfactual outcome distributions $\hat{\theta}=(\hat{\beta}^{\top},\hat{\delta}^{\top})^{\top}$ at a few grid points $y$ in $(-2.0,-1.5,-1.0,-0.5,1.0,0.5)$. 
The true parameter values are
\begin{align*}
\beta_0^y&=\frac{\int_{0}^{1}\{\Phi(y-2x^{2})(x-1)+\Phi(y)\}\mathrm{d}x}{\int_{0}^{1}x\mathrm{d}x},\quad
\delta_0^y=\frac{\int_{0}^{1}\{\Phi(y-2x^{2})x\}\mathrm{d}x}{\int_{0}^{1}x\mathrm{d}x}.
\end{align*}

\section{Application details}\label{sec:application}

 \cite{angrist1998children} estimate the impact of childbearing $D$ on female labor supply $Y$ in a sample of 394,840 mothers, aged 21--35 with at least two children, from the 1980 Census. The first instrument $Z_1$ is twin births: $Z_1$ indicates whether the mother's second and third children were twins. The second instrument $Z_2$ is same-sex siblings: $Z_2$ indicates whether the mother's initial two children were siblings with the same sex. The authors reason that both $(Z_1,Z_2)$ are quasi random events that induce having a third child such that the independence assumption holds unconditionally. However, the instruments are not independent of $X$, and therefore $\pi_{0}(X)$  still depends on $X$ and may be estimated.
 
 \cite{angrist2013extrapolate}  use parametric $\kappa$ weights to estimate two complier characteristics: (i) the average age of the mother's second child; and (ii) the years of schooling of the mother. For a given characteristic $f(X)=X$, the authors specify the instrument propensity score model as
 $$
 \pi_{0}(X)=[1+\exp\{-(\beta_0+\beta_1X+\beta_2X^2+\beta_3X^3+\beta_4X^4)\}]^{-1}.
 $$
 As discussed in Section~\ref{sec:insight}, such an approach is only valid when the parametric assumption on $\pi_{0}(X)$ is correct.  
 
 The semiparametric Auto-$\kappa$ approach we propose combines the doubly robust moment function from Theorems~\ref{thm:general_kappa} and~\ref{thm:general_kappa3} with the meta procedure in Algorithm~\ref{alg_dml} and the regularized balancing weights in Algorithm~\ref{alg_stage1}. We expand the dictionary of basis functions to include sixth order polynomials of $X$, and interactions between $Z$ and polynomials of $X$. We directly estimate and regularize both the regression $\hat{\gamma}$ and the balancing weights $\hat{\alpha}$, tuning the regularization according to Algorithm~\ref{alg_tuning}. We set the hyperparameters  $(c_1,c_2,c_3)$ as  $(0.5,0.1,0.1)$. In sample splitting, we partition the sample into five folds. The estimated balancing weights $\hat{\alpha}$ imply extreme twins-instrument propensity scores for a few observations.  We censor the extreme propensity scores by setting the implied $\hat{\pi}<10^{-12}$ to be $10^{-12}$ and $\hat{\pi}>1-10^{-12}$ to be $1-10^{-12}$.

 Finally, as a robustness check, we verify that $\kappa$ weighting and Auto-$\kappa$ yield similar estimated shares of compliers, i.e. similar estimates of $\text{\normalfont pr}\{D^{(1)}>D^{(0)}\}$. These shares are typically reported in empirical research to interpret the strength and relevance of an instrumental variable. In the language of two stage least squares, these estimates correspond to the first stage. Table~\ref{tab:2} reports the complier share estimates underlying the results of Table~\ref{tab:1}. Auto-$\kappa$ produces similar complier share estimates to the $\kappa$ weight approach of \cite{angrist2013extrapolate} while allowing for more flexible models and regularization.
 
\begin{table}[h]
\caption{Comparison of complier shares}
\begin{centering}
\begin{tabular}{lcccc}
 \\
 & \multicolumn{2}{c}{Average age of second child} & \multicolumn{2}{c}{Average schooling of mother} \\    
 & Twins & Same-sex   & Twins & Same-sex  \\[5pt]
 $\kappa$ weight & 0.60 & 0.06 &   0.60 & 0.06   \\
Auto-$\kappa$ & 0.73 & 0.06 &  0.76 & 0.06   
\end{tabular}\\
\end{centering}
\label{tab:2}
\end{table}

\end{document}